\documentclass{ieeetmlcn}
\usepackage{cite}
\usepackage{amsmath,amssymb,amsfonts}
\usepackage{graphicx,color}
\usepackage{textcomp}
\usepackage{xcolor}
\usepackage{hyperref}
\hypersetup{hidelinks=true}
\def\BibTeX{{\rm B\kern-.05em{\sc i\kern-.025em b}\kern-.08em
    T\kern-.1667em\lower.7ex\hbox{E}\kern-.125emX}}
\AtBeginDocument{\definecolor{tmlcncolor}{cmyk}{0.93,0.59,0.15,0.02}\definecolor{NavyBlue}{RGB}{0,86,125}}

\def\OJlogo{\vspace{-4pt}\includegraphics[height=18pt]{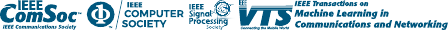}}
\def\seclogo{\vspace{10pt}\includegraphics[height=18pt]{new_logo-4.pdf}}

\def\authorrefmark#1{\ensuremath{^{\textbf{#1}}}}

\usepackage{xcolor}
\usepackage{booktabs}
\usepackage{multirow}
\usepackage{amsmath,amsfonts,bm}
\let\labelindent\relax
\usepackage{enumitem}
\usepackage{mathtools}
\usepackage{algorithm}
\usepackage{algpseudocode}
\usepackage{hyperref}
\usepackage{caption}
\usepackage{subcaption}
\usepackage{tikz} 
\usetikzlibrary{positioning}
\usepackage{amsthm}
\usepackage{amsfonts}
\usetikzlibrary{external}
\usepackage{tabularx}
\newtheorem{theorem}{Theorem}[section]
\newtheorem{lemma}[theorem]{Lemma}
\newtheorem{corollary}{Corollary}[theorem]

\newcommand{\ie}{{\em i.e., }}
\newcommand{\eg}{{\em e.g., }}

\renewcommand\vec{\mathbf}
\newcommand{\xbar}[1][t]{\bar{\vec{x}}_{#1}}
\newcommand{\xv}[1][t]{\vec{x}^v_{#1}}
\newcommand{\xtild}[1][t]{\tilde{\vec{x}}_{#1}}

\DeclareMathOperator{\EX}{\mathbb{E}}
\newcommand{\norm}[1]{\| #1 \|^2}
\newcommand{\normm}[1]{\| #1 \|}
\newcommand{\inpr}[2]{\langle #1 , #2 \rangle}
\newcommand{\fiv}[1][t]{f_{i_{#1}^v}}
\newcommand{\xz}{\vec{x}_0}
\newcommand{\xstar}{\vec{x}^*}
\newcommand{\xhat}{\hat{\vec{x}}_T}
\newcommand{\sumov}{\sum_{v=1}^{V}}

\newcommand{\gstar}{\vec{g}^*_t}
\newcommand{\g}{\vec{g}_t}
\newcommand{\gbar}{\bar{\vec{g}}_t}
\newcommand{\DvtoD}[1][v]{\frac{D_{#1}}{D}}
\newcommand{\data}[1]{\mathcal{D}_#1}

\newcommand{\lv}[1][t]{l^v_{#1}}
\newcommand{\uv}[1][t]{u^v_{#1}}
\newcommand{\sv}[1][t+1]{s^v_{#1}}
\newcommand{\tav}[1][t]{\tau^v_{#1}}
\usepackage{xspace,exscale,relsize}

\newcommand{\ours}{DIGEST\xspace}
\newcommand{\gos}{Gossip\xspace}

\usepackage{hyperref}
\usepackage{url}

\begin{document}


\title{DIGEST: Fast and Communication Efficient  Decentralized Learning with Local Updates}

\author{Peyman Gholami\authorrefmark{1}, Member, IEEE, Hulya Seferoglu\authorrefmark{1}, Senior Member, IEEE}
\affil{Department of ECE, University of Illinois at Chicago}
\corresp{Corresponding author: Peyman Gholami (email: pghola2@uic.edu).}


\begin{abstract}

Two widely considered decentralized learning algorithms are \gos and random walk-based learning. \gos algorithms (both synchronous and asynchronous versions) suffer from high communication cost, while random-walk based learning experiences increased convergence time. In this paper, we design a fast and communication-efficient asynchronous decentralized learning mechanism \ours by taking advantage of both \gos and random-walk ideas, and focusing on stochastic gradient descent (SGD). 
%
\ours is an asynchronous decentralized algorithm building on local-SGD algorithms, which are originally designed for communication efficient centralized learning. We design both single-stream and multi-stream \ours, where the communication overhead may increase when the number of streams increases, and there is a convergence and communication overhead trade-off which can be leveraged. We analyze the convergence of single- and multi-stream \ours, and prove that both algorithms approach to the optimal solution asymptotically for both iid and non-iid data distributions. We evaluate the performance of single- and multi-stream \ours for logistic regression and a deep neural network ResNet20.  The simulation results confirm that multi-stream \ours has nice convergence properties; \ie its convergence time is better than or comparable to the baselines in iid setting, and outperforms the baselines in non-iid setting.

%
%
%
%
%
%
\end{abstract}



\maketitle

\vspace{-10pt}
\section{Introduction}
Emerging applications such as Internet of Things (IoT),  mobile healthcare, self-driving cars, etc. dictate learning be performed on data predominantly originating at edge and end user devices \cite{gubbi2013internet, li2018learning,iot}. 
A growing body of research work, \eg federated learning \cite{FL, kairouz2021advances, Konecn2015FederatedOD,Mc2017FL,Li2019FL,Li2020OnTC} has focused on engaging the edge in the learning process, along with the cloud, by allowing the data to be processed locally instead of being shipped to the cloud.  
Learning beyond the cloud can be advantageous in terms of better utilization of network resources, delay reduction, and resiliency against cloud unavailability and catastrophic failures.
%
However,  the proposed solutions, like federated learning,  predominantly suffer from having a critical centralized component referred to as the Parameter Server (PS) organizing and aggregating the devices' computations. Decentralized learning, advocating the elimination of PSs, emerges as a promising solution to this problem. 

Decentralized algorithms have been extensively studied in the literature, with \gos algorithms receiving the lion's share of research attention \cite{Boyd2006RandomizedGA, Nedic2009DistributedSM,Koloskova2019DecentralizedSO,Aysal2009BroadcastGA,Duchi2012DualAF, kempe2003Gossip,Xiao2003FastLI,Boyd2006RGossip}. In \gos algorithms, each node (edge or end user device) has its own locally kept model on which it effectuates the learning by talking to its neighbors. 
This makes \gos attractive from a failure-tolerance perspective. However, this comes at the expense of high network resource utilization. 
%
%
{As shown in Fig. 1a, nodes in the synchronous \gos algorithm use a synchronous clock to perform local model update and aggregation where aggregation demands receiving model updates from the neighbors.
Until their synchronization clocks expire, the nodes receive model updates from their neighbors. As seen, there should be data communication among all nodes after each model update, which is a significant communication overhead.
Furthermore, some nodes may be a bottleneck for the synchronization as these nodes (which are also called stragglers) can be delayed due to computation and/or communication delays, which increases the convergence time. This is due to the synchronous clock time that is determined according to the slowest node (or a set of fastest nodes).   
}
%

\begin{figure*}[thb]
     \centering
     \begin{subfigure}[b]{0.46\textwidth}
         \centering
         \includegraphics[width=\textwidth]{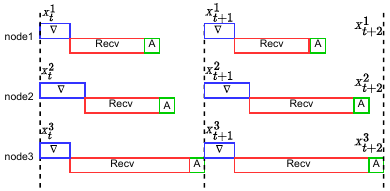}
         \caption{Sync-\gos}
         \label{syncgos}
     \end{subfigure}
     \hfill
     \begin{subfigure}[b]{0.46\textwidth}
         \centering
         \includegraphics[width=\textwidth]{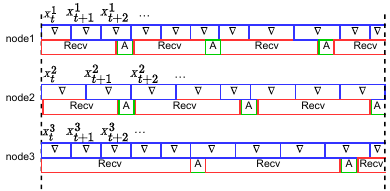}
         \caption{Async-\gos}
         \label{asyncgos}
     \end{subfigure}
     \begin{subfigure}[b]{0.46\textwidth}
         \centering
         \includegraphics[width=\textwidth]{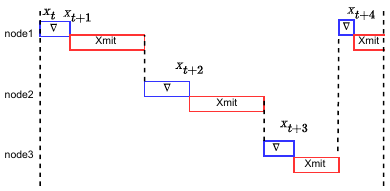}
         \caption{Random-Walk}
         \label{RW}
     \end{subfigure}
     \hfill
     \begin{subfigure}[b]{0.46\textwidth}
         \centering
         \includegraphics[width=\textwidth]{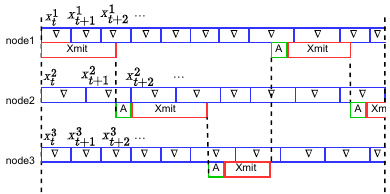}
         \caption{\ours}
         \label{digest}
     \end{subfigure}
        \caption{\ours in perspective as compared to existing decentralized learning algorithms; (a) synchronous \gos, asynchronous \gos, and random-walk. Note that ``$\nabla$'' represents a model update. ``Xmit'' represents the transmission of a model from a node to one of its neighbors. ``Recv'' represents the communication duration while receiving model updates from all of a node's neighbors.  
        ``A'' represents model aggregation. $\xv[t]$ shows the local model of node $v$ at iteration $t$. For random walk algorithm, the global model iterates are denoted as $\vec{x}_t$. {We note that the absence of blue boxes in all figures means that nodes do not continue their computations. On the other hand, the absence of red boxes means that there is no communication among neighboring nodes. We also note that communication (``Xmit'') and computation (``$\nabla$'')  are parallel in \ours and asynchronous \gos, but aggregation (``A'') and computation are sequential. The figure shows them as parallel tasks for the sake of easier presentation and considering that the duration of aggregation (``A'') is negligible as compared to communication (``Xmit'') and computation (``$\nabla$'').}
        }  
        \label{diagram}
    \vspace{-5pt}
\end{figure*}

Asynchronous \gos algorithms, where nodes communicate asynchronously and without waiting for others are promising to reduce idle nodes and 
eliminate the stragglers, \ie delayed nodes \cite{Lian2018Async,Assran2019StochasticGP, Li2018PipeSGDAD, Avidor2022LocallyAS}.
Indeed, asynchronous algorithms significantly reduce the idle times of nodes by performing model updates and model exchanges simultaneously as illustrated in Fig. \ref{asyncgos}.
For example, node 1 can still update its model from $\vec{x}_t^1$ to $\vec{x}_{t+1}^1$ and $\vec{x}_{t+2}^1$ while receiving model updates from its neighbors.
When it receives from all (or some) of its neighbors, it performs model aggregation.
{However, nodes still rely on iterative \gos averaging of their models, so  
updates propagate gradually across the network.}
Such delayed updates, also referred as gradient staleness in asynchronous \gos may lead to high error floors \cite{Dutta2021SlowAS}, or require very strict assumptions to converge to the  optimum solution \cite{Lian2018Async}. 
Moreover, such methods must be implemented with caution to prevent the occurrence of deadlocks \cite{Assran2019StochasticGP}.

{In both synchronous and asynchronous \gos, models propagate over the nodes and are updated by each node gradually as seen in Fig. \ref{infoa}. This may lead to a notion that we name ``diminishing updates'', where a node's update (e.g., node 1 in Fig. \ref{infoa}), even though crucial for convergence, may be averaged and mixed with other models in the next node (e.g., node 2 in Fig. \ref{infoa}). The diminishing updates are more emphasized when a model passes through high degree nodes, and detrimental to the convergence when data distribution is heterogeneous across the nodes.  }


If \gos algorithms are one side of the spectrum of decentralized learning algorithms, the other side is random-walk based decentralized learning \cite{Bertsekas96anew,GhadirRW2021, Sun2018MarkovSGD, Needell2014WeightedSampling}.
The random-walk algorithms advocate activating a node at a time, which would update the global model with its local data as illustrated in Fig. \ref{RW}. Then, the node selects one of its neighbors randomly and sends the updated global model. The selected neighbor becomes a newly activated node, so it updates the global model using its local data.
This continues until convergence. Random-walk algorithms reduce the communication cost as well as computation with the price of increased convergence time due to idle times at nodes. 

The goal of this work is to take advantage of both \gos and random-walk ideas to design a fast and communication-
efficient decentralized learning. 
Our key intuitions are; (i) Nodes do not need to communicate as much as \gos to update their models, i.e., a sporadic exchange of model updates is sufficient; {(ii) the diminishing updates inherent to \gos algorithms can be eliminated by employing a global model (as shown in Fig. \ref{infob}) as nodes do not average out multiple models received from multiple neighbors; they only add their models to the global model}; 
and (iii) nodes do not need to wait idle as in random walk. 

We design a fast and communication-efficient asynchronous decentralized learning mechanism \ours by particularly focusing on stochastic gradient descent (SGD).
\ours is an \emph{asynchronous decentralized learning} algorithm building on local-SGD algorithms, which are originally designed for communication efficient \emph{centralized learning} \cite{Stich2019LSGD, Wang2021LocalSGD, Lin2020DontUM}. 
%
%
%
In local-SGD, each node performs multiple model updates before sending {the model} to the PS. 
The PS aggregates the updates received from multiple nodes and transmits the updated global model back to nodes.
The sporadic communication between nodes and the PS reduces the communication overhead. We exploit this idea for \emph{asynchronous decentralized learning}. The following are our contributions.

\begin{itemize}[leftmargin=*]
\vspace{-5pt}
   \item \textbf{Design of \ours.} We design a fast and communication-efficient  asynchronous decentralized learning mechanism; \ours by particularly focusing on stochastic gradient descent (SGD).
   \ours works as follows. Each node keeps updating its local model all the time as in local-SGD. Meanwhile, there is an ongoing stream of global model update among nodes, Fig. \ref{digest}. For example, node 1 starts transmitting the global model to node 2 at time $t$. When node 2 receives the global model from node 1, it aggregates it with its local model. The aggregated global model is transmitted to node 3 next.
   We note that the exchanged models are global models as each node adds its own local updates to the received model.
   A node that has the global model selects the next node randomly among its neighbors for global model transmission.  
   After all the nodes update their {models} with a global model, \ours pauses global model exchange, while local SGD computations still continue. The global model exchange is repeated at every $H$ iterations. We name this algorithm single-stream \ours.  

\begin{figure}[t]
     \centering
     \begin{subfigure}[b]{0.22\textwidth}
         \centering
         \scalebox{0.75}{
         \begin{tikzpicture}[main/.style = {draw, circle,minimum size = 0.3cm,inner sep=0pt}]
            \node[main] (1) {};
            \node[above = .15] at (1) {node 1};
            \node[main] (2) [below = 1.cm of 1] {};
            \node[left = .3] at (2) {node 2};
            \node[main] (3) [below = 1.2cm of 2] {};
            \node[below = .15] at (3) {node 3};
            \node (21) [above left= .7cm and .5 cm of 2] {};
            \node (22) [below left= .7 cm and .5cm of 2] {};
            \node (27) [above left= .3cm and .9 cm of 2] {};
            \node (28) [below left= .3cm and .9 cm of 2] {};
            \node (29) [above right= .3cm and .9 cm of 2] {};
            \node (30) [below right= .3cm and .9 cm of 2] {};
            \node (23) [above right= .7cm and .5 cm of 2] {};
            \node (24) [below right= .7 cm and .5cm of 2] {};
            \draw[->,line width=.5mm,dotted] (21) -- (2);
            \draw[->,line width=.5mm,dotted] (22) -- (2);
            \draw[->,line width=.5mm,dotted] (23) -- (2);
            \draw[->,line width=.5mm,dotted] (24) -- (2);
            \draw[->,line width=.5mm,dotted] (27) -- (2);
            \draw[->,line width=.5mm,dotted] (28) -- (2);
            \draw[->,line width=.5mm,dotted] (29) -- (2);
            \draw[->,line width=.5mm,dotted] (30) -- (2);
            \node (a) [rectangle, draw, text = olive, fill = blue!100] [right = .4] at (1){};
            \node (b) [rectangle, draw, text = olive, fill = blue!15] [right = .4] at (2){};
            \node (c) [rectangle, draw, text = olive, fill = blue!15] [right = .4] at (3){};
            \draw[->,line width=.5mm,dotted] (1) -- (2);
            \draw[->,line width=.5mm,dotted] (2) -- (3);
            \draw[->,line width=.7mm] (a) -- (b);
            \draw[->,line width=.7mm] (b) -- (c);
        \end{tikzpicture}}
         \caption{\gos}
         \label{infoa}
     \end{subfigure}
     \begin{subfigure}[b]{0.22\textwidth}
     \centering
        \scalebox{0.75}{
         \begin{tikzpicture}[main/.style = {draw, circle,minimum size = 0.3cm,inner sep=0pt}]
            \node[main] (1) {};
            \node[above = .15] at (1) {node 1};
            \node[main] (2) [below = 1.2cm of 1] {};
            \node[left = .15] at (2) {node 2};
            \node[main] (3) [below = 1.2cm of 2] {};
            \node[below = .15] at (3) {node 3};
            \node (a) [rectangle, draw, text = olive, fill = blue!100] [right = .4] at (1){};
            \node (b) [rectangle, draw, text = olive, fill = blue!100] [right = .4] at (2){};
            \node (c) [rectangle, draw, text = olive, fill = blue!100] [right = .4] at (3){};
            \draw[->,line width=.5mm,dotted] (1) -- (2);
            \draw[->,line width=.5mm,dotted] (2) -- (3);
            \draw[->,line width=.7mm] (a) -- (b);
            \draw[->,line width=.7mm] (b) -- (c);
        \end{tikzpicture}}
        \caption{\ours}
        \label{infob}
     \end{subfigure}
        \caption{Spread of information across a decentralized network.}\label{info-spread}
        \vspace{-15pt}
\end{figure}
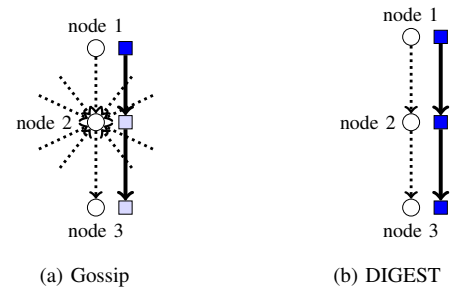

\item \textbf{Multi-Stream \ours.}   We further improve the convergence time of single-stream \ours by enabling multiple streams of global model updates.
   For example, there are $4$ streams working together at Fig. \ref{fig:multi_st}, where
   each stream operates on a smaller part of the network, so global model updates can be completed quickly. We identify the multiple streams using a rooted tree, which is determined in a decentralized manner via a distance vector algorithm \cite{Hedrick1988RoutingIP}. Note that two or more streams may intersect at one node which is how the streams aggregate their global models.
   The communication overhead may increase when the number of streams increases, and there is a nice convergence and communication overhead trade-off which can be exploited by adjusting $H$.
    \item \textbf{Convergence analysis of \ours.} 
    %
    We analyze the convergence of single- and multi-stream \ours, and prove that both algorithms approach the optimal solution asymptotically. We show that \ours's approach of simultaneous global model updating and local SGD iterations does not hurt the convergence rate and does not create any convergence gap. Furthermore, \ours is not affected by the network topology, \ie even high degree nodes do not create any learning bias. 
    We also indicate how frequently global model updates should be made, i.e., what the value of $H$ should be to achieve linear speed-up $O(\frac{1}{VT})$ in both iid and non-iid cases, where $V$ is the number of nodes in the network and $T$ is total number of iterations. 
    
    

    \item \textbf{Evaluation of \ours.} 
    We evaluate the performance of single- and multi-stream \ours for logistic regression and a deep neural network ResNet20 \cite{He2015resnet} for datasets \textit{w8a} \cite{w8a} and \textit{MNIST} \cite{mnist}, and \textit{CIFAR-10} \cite{cifar10}. We consider both iid and non-iid data distributions over various network topologies with different number of nodes. 
    %
    %
    The simulation results confirm that \ours has nice convergence properties; \ie its convergence time is better than or comparable to the baselines in iid setting, and outperforms the baselines in non-iid setting.
    

    

\end{itemize}

\begin{figure}[t]
     \centering
        \scalebox{0.75}{
         \begin{tikzpicture}[main/.style = {draw, circle,minimum size = 0.3cm,inner sep=0pt}]
            \node[main] (1) {};
            \node[left=.1] at (1) {node 1};
            \node[main] (2) [above right=0.7cm and 2cm of 1] {};
            \node[above right = .01] at (2) {node 2};
            \node[main] (5) [below right=0.7cm and 2cm of 2] {};
            \node[right=.2] at (5) {node 5};
            \node[main] (4) [below left=1cm and .2cm of 2] {};
            \node[above left =.09 and .001] at (4) {node 4};
            \node[main] (3) [below left=.7cm and 1cm of 4] {};
            \node[left=.1] at (3) {node 3};
            \node[main] (6) [below right=.7cm and 1cm of 4] {};
            \node[below = .1] at (6) {node 6};
            \node[main] (7) [below right=.1cm and 1.5cm of 6] {};
            \node[right=.2] at (7) {node 7};
            \draw[line width=.5mm] (1) -- (2);
            \draw[line width=.5mm] (1) -- (3);
            \draw[line width=.5mm] (2) -- (5);
            \draw[line width=.5mm] (4) -- (5);
            \draw[line width=.5mm] (4) -- (3);
            \draw[line width=.5mm] (4) -- (6);
            \draw[line width=.5mm] (1) -- (4);
            \draw[line width=.5mm] (3) -- (6);
            \draw[line width=.5mm] (6) -- (5);
            \draw[line width=.5mm] (7) -- (6);
            \draw[line width=.5mm] (7) -- (5);
            \draw[<->,red,line width=.5mm,dotted] (1.north) .. controls  +(north east:1.6)  and +(north west:1.2)   .. (5) node[midway,above left = .1 and .1] {stream 1};
            \draw[<->,purple,line width=.5mm,dotted] (3) .. controls  +(south east:1cm )  and +(south west:1.2)   .. (5) node[midway, below left = .2] {stream 3};
            \draw[<->,blue,line width=.5mm,dotted] (4) .. controls  +(north east:.5)  and +(north west:.5)   .. (5) node[midway, above left= .001] {stream 2};
            \draw[<->,cyan,line width=.5mm,dotted] (7) .. controls  +(north east:.5)  and +(south east:.5)   .. (5) node[midway, right] {stream 4};
        \end{tikzpicture}
        }
    \caption{Example multi-stream \ours.}
    \label{fig:multi_st}
    \vspace{-10pt}
\end{figure}
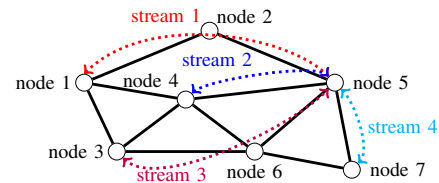

\section{Related Work}

Decentralized optimization algorithms have been widely studied in the literature, where nodes interact with their neighbors to solve an optimization problem \cite{Tsits1984DCN, Nedic2009DisSubGrad, Chen2012Diffusion, Duchi2012DualAveraging}. Despite their potential, these algorithms suffer from a bias in non-iid data \cite{yuan2016OnConverge}, and they require synchronization and  orchestration among nodes, which is costly in terms of communication overhead. 



Decentralized algorithms based on \gos involve a mixing step where nodes compute their new models by mixing their own and neighbors' models \cite{Koloskova2020Unified,Kevin2019OptimalConverge,Xiao2003FastLI}. However, this is costly in terms of communication as every node requires $O(\deg(G))$ data exchange for every model update for a graph $G$. 
Furthermore, model updates propagate gradually over the network due to iterative gossip averaging. Such gradual model propagation reduces the convergence time and makes the learning mechanism very sensitive to data distribution over nodes. 
%
%
Finally, \gos algorithms tend to favor higher-degree nodes while updating models, which causes slower convergence for non-iid data \cite{Giaretta2019BeatenPath}.
Our goal in this paper is to reduce the communication cost in decentralized learning for any data distribution without hurting convergence.

Asynchronous \gos algorithms are designed to improve synchronous \gos algorithms. The main focus of asynchronous \gos is the reduction of synchronization costs in a gossip setting by utilizing non-blocking communication \cite{Lian2018Async,Assran2019StochasticGP,swarm}. This means that nodes could potentially receive and use the stale versions of neighboring nodes' models to update their own models. Despite this modification, asynchronous \gos algorithms continue to rely on the traditional \gos to spread the information, where each node sends its model to all of its neighbors, which still introduces high communication cost. Our goal in this paper is to reduce the communication cost in an asynchronous manner for decentralized learning.  

A random walk-based decentralized learning is considered in \cite{GhadirRW2021}, which is similar to work on random walk data sampling for stochastic gradient descent, \eg \cite{Sun2018MarkovSGD, Needell2014WeightedSampling}. 
Reducing the global averaging rounds as compared to \gos-based mechanisms is considered in \cite{spiridonoff2021communicationefficient} by one-shot averaging. However, the global averaging rounds require long synchronization duration for large networks, which increases the convergence time. Also, strong assumptions and only iid data is considered \cite{spiridonoff2021communicationefficient}. As compared to \gos and random walk-based algorithms, \ours designs a communication efficient decentralized learning without hurting convergence rate for both iid and non-iid data.





\section{\label{sec:digest}Design of \ours}
\subsection{Preliminaries}
\textbf{Network Topology.} We model the underlying network topology with a connected graph $G=(\mathcal{V},\mathcal{E})$, where $\mathcal{V}$ is the set of vertices (nodes) and $\mathcal{E}$ is the set edges.
The vertex set contains $V$ nodes, \ie $ |\mathcal{V}|= V$, and $|.|$ shows the size of the set.
The computing capabilities of nodes are arbitrary and heterogeneous. 
If node $i$ is connected to node $j$ through a communication link and can transmit data, then link $\{i, j\}$ is in the edge set, \ie $\{i, j\} \in \mathcal{E}$.
The set of the nodes that node $i$ is connected to and can transmit data is called the neighbors of node $i$, and the neighbor set of node $i$ is denoted by $\mathcal{N}_i$.
We do not make any assumptions about the behavior of the communication links; there can be an arbitrary, but finite amount of delay over the links.



\textbf{Data.} We consider a setup where nodes have access to a subset of data samples $\mathcal{D}$.
Each node $v$ has a local dataset $\data{v}$, where $D_v = |\data{v}|$ is the size of the local dataset and $D=\sumov D_v$.
The distribution of data across nodes is not identical and independently distributed (non-iid). 

\textbf{Stochastic Optimization.} Assume that the nodes in the network jointly minimize a $d$-dimensional function  $f: \mathbb{R}^d \rightarrow \mathbb{R}$.
The goal of the nodes is to converge on a model $\xstar$, which minimizes the empirical loss over $D$ samples, \ie $ \xstar := \arg \min_{\vec{x} \in 	\mathbb{R}^d}\Big[f(\vec{x}) := \frac{1}{D} \sum_{i=1}^D f_i(\vec{x})\Big]\label{e1}$,
where $f_i (\vec{x}): \mathbb{R}^d \rightarrow \mathbb{R}$ is the loss function of $\vec{x}$ associated with the data sample $i$.
The optimum solution is denoted by $f^*$.
The loss function on local dataset $\data{v}$ at node $v$ is $f^v(\vec{x}) = \frac{1}{D_v} \sum_{i \in \data{v}} f_i(\vec{x})$. 

\textbf{Notation.} We provide our notation table in Appendix \ref{appendixa}. 

%





\subsection{Single-Stream \ours}


\subsubsection{Overview} 
\textbf{Local Model Update.} We assume that the time is slotted, and at each slot/iteration, a local model  is updated.
However, a calculation of a gradient may take more than one slot, vary over time, or not fit into slot boundaries. 
Thus, at each iteration $t$, any gradients which have been delayed up to iteration $t$, and not used in previous local updates are used to update the local model.
We note that time slots across nodes do not need to be synchronized in \ours as each node can have its own iteration sequence and update local and global models over its own sequence. The only assumption we make is that the slot sizes are the same across nodes, which can be decided a priori.

Let us consider that ${L}^v_T=\{\lv\}_{0\leq t <T}$ is the set of 
the delayed gradient calculations at node $v$, where $\lv$ shows that the local-SGD update of iteration $t$ is delayed until iteration $\lv$.
For instance, $\lv[t']=t$ means that the local-SGD of iteration $t'$ is lagged behind and performed in iteration $t$, $t\geq t'$.
Then, we define $\uv = \{ t' \mid \lv[t']=t\}$ to show all the updates completed at iteration $t$ in node $v$.
If we consider that there is no global update at node $v$, the local model is updated as $\xv[t+1] = \xv- \sum_{z\in \uv} \eta_{z} \nabla \fiv[z](\xv[z])$, where $\eta_{z}$ is the learning rate, $i^v_z$ is a sample uniformly chosen from $\data{v}$ in iteration $z$, and $\nabla \fiv[z](\xv[z])$ is the gradient.
However, there may be global model updates at node $v$, \ie node $v$ could receive a global model update from one of its neighbors at iteration $t$.
Such a global model reception should be reflected in local model updates, which we discuss next.

\textbf{Global Model Update and Exchange.} 
Let $\xtild$ be the global model that is being transferred from one node to another at time slot $t$. If node $v$ receives the global model $\xtild$ from one of its neighbors, a global model update indicator $s_{t}^v$ is set to $s_{t}^v = 1$. Otherwise, \ie when node $v$ does not receive the global model from its neighbors, we set $s_{t}^v = 0$.  

If $s_{t}^v = 0$, then node $v$ updates its model locally according to the update mechanism presented earlier in the ``Local Model Update'' section. If $s_{t}^v = 1$, \ie when a global model is received by node $v$ from one of its neighbors, then the global model should be incorporated in the calculations.
\ours sets the local model to the global model when there is a global model update as follows.
\begin{equation}
\xv[t]=\begin{cases}
          \xv[t-1]- \sum_{z\in \uv[t-1]} \eta_{z} \nabla \fiv[z](\xv[z]) \quad &\text{if} \, \sv[t]=0\\
          \xtild[t] \quad &\text{if} \, \sv[t]=1 \\
     \end{cases} \label{basic}
\end{equation} 

The global model is updated as
\begin{align}\label{eq:globMod1}
\xtild &= \xtild[t-1]+\DvtoD\bigg( \hspace{-.25em}\big(\xv[t-1]- \hspace{-.5em}\sum_{z\in \uv[t-1]} \hspace{-.5em} \eta_{z} \nabla \fiv[z](\xv[z])\big) - \xv[\tau^v_{t-1}]\bigg),
\end{align}
where $\xtild[t-1]$ is the global model received by node $v$ at slot $t-1$. The global model, \ie $\xtild$ is updated by using  $\xtild[t-1]$ as well as the local updates of node $v$. 
We use $\tav$ to denote the last time slot up to $t$, when node $v$'s model was updated with the global model, \ie $\tau^v_t = \max \{t'\mid  t'\leq t, s_{t'}^v=1\}$. The equivalent of (\ref{eq:globMod1}) is $\xtild =\xz - \sumov \sum_{t'=0}^{\tav -1} \sum_{z \in \uv[t']} \DvtoD \eta_z \nabla \fiv[z](\xv[z])$, 
where $\vec{x}_0$ is the initial model. As seen, the global model is updated across all nodes by taking into account all delayed gradient calculations. We use $\DvtoD$ ratio to give more weight to the gradients with larger data sets. Now that we provided an overview of \ours, we provide details on how \ours algorithms operate next. 
\begin{algorithm}[t!]
    \caption{ Local and global model update of \ours at node $v \in \mathcal{V} $.}\label{alg:Dec}
\begin{algorithmic}[1]
    \State
    \textbf{Initialization: }$\xv[0]=\xz$, $\xv[-1]=\xz$, $\xtild[0]=\xz$, $visited=\{\}$, $pre\_node = v$, $\mathcal{S}^v_T = \{0\}_{0<t\leq T} $, $s_1^{v_0}=1$, $\mathcal{V}_0=\mathcal{V}$.
\For{$t$ in $0,...,T-1$}
\State Sample $i^v_t$ uniformly from $\data{v}$.
\State Compute the gradient $\nabla \fiv(\xv)$.
\State $\xv[t+1] = \xv- \sum_{z\in \uv} \eta_{z} \nabla \fiv[z](\xv[z])$ \Comment{Local model update.}
\If{Received new $message$ from another node}
    \State $(\xtild[t], visited, pre\_node, 0) \leftarrow message$
    \State  $s^v_{t+1}=1$
\EndIf
\If{$s^v_{t+1}=1$} 
    \State $\xtild[t+1] = \xtild[t] + \frac{D_v}{D} (\xv[t+1]-\xv[-1])$
    \State $\xv[t+1] = \xtild[t+1]$  \Comment{Local model is updated using global model.}
    \State $\xv[-1]=\xv[t+1]$ 
    \If{$\mod(t,H)=0$ or $visited \neq \mathcal{V}$}
    \State Send $message$ $=$  $(\xtild[t+1],$ $visited,$ $ pre\_node,0)$ to a neighbor node via Alg.  \ref{alg:traverse}
    %
    \Else
    \State $s^v_{t+H-mod(t,H)}=1$ \Comment{Pause global model update and exchange until $\mod(t,H)=0$ holds. 
    }
    \State $visited=\{\}$
    \EndIf
\EndIf
\EndFor
\end{algorithmic}
\end{algorithm}

\subsubsection{Algorithm Design} 

\ours is comprised of two algorithms; (i) local and global model update at node $v$, and (ii) sending a global model from a node to its neighbor.

\textbf{Local and Global Model Update.} The local and global model update of \ours is presented in  Alg. \ref{alg:Dec}. Every node $v$ keeps its local model $\xv$ as well as $\xv[-1]$, which is a copy of the local model in the latest global model update at node $v$.
$\xtild$ is the global model. All of these models are initialized with the same initial model $\xz$. We note that only one of the nodes, let us say node $v_0$, has the global model $\xtild$ at the start of the algorithm.

We define $visited$ as the set of nodes that are recently visited for the global model updates.
It is initialized as an empty set at node $v$. We define a period of time, during which all the nodes in $\mathcal{V}$ are visited at least once, as a synchronization round.
During a synchronization round, all nodes update their local models with a global model as they are visited at least once. More details regarding the $visited$ set will be provided as part of Alg. \ref{alg:traverse}.

The node that node $v$ receives the global model from is defined by $pre\_node$, where its initial value is set to $v$ as there is no previous node at the start. The set of global model update indicators, \ie $\mathcal{S}^v_T = \{\sv[t]\}_{0<t\leq T} $ is initialized as an empty set, where $T$ is the number of slots that Alg. \ref{alg:Dec} runs. Assuming that $v_0$ is the node where the global model update starts, $s_1^{v_0}$ is set to 1, \ie $s_1^{v_0}=1$. 




\begin{algorithm}[t!]
\caption{Sending global model from node $v \in \mathcal{V}$.}\label{alg:traverse} 
\begin{algorithmic}[1]
\State \textbf{Input: } $message = (\xtild[t+1],visited, pre\_node,m)$ 
\If{$v \notin visited$}
    \State $visited = visited \cup \{v\}$
    \State $p^v_m=pre\_node$
\EndIf
\State $\mathcal{C}=\{ v' \in \mathcal{N}_v \cap \mathcal{V}_m\mid v' \notin visited \}$
\If{$\mathcal{C} \neq \emptyset$}
    \State Select $v'$ randomly from $\mathcal{C}$.
    \State Send $message = (\xtild[t+1], visited, v, m)$ to node $v'$.
    \Else
    \State Send $message = (\xtild[t+1], visited, v, m)$ to $p^v_m$.
    \EndIf
\end{algorithmic}
\end{algorithm}

At every iteration $t$, node $v$ first gets one data sample from the local dataset  randomly (line $3$), and computes a stochastic local gradient (line $4$) based on the selected data sample and the current model at node $v$, \ie $\xv$. Then, node $v$ uses all the gradients whose computations are delayed until iteration $t$, and that are not used in local model updates so far for the local model update (line $5$). 

If node $v$ receives a ``$message$'' from one of its neighbors at slot $t$, then it should update the global model. Each $message$ contains information on the global model $\xtild$, the set of visited nodes, \ie $visited$, the id of the node that sends this message to node $v$, \eg $v'$, and a parameter $m$, which is always set to $0$ in single-stream \ours, but may take different values for multi-stream \ours.
After the message is extracted (line $7$), the global model update indicator is set to $1$ (line $8$), and the global model is updated (lines $11 - 13$).
The global model is updated using the most recent local model of node $v$ (line $11$). The local model is updated with the global model (line $12$). The current local model is stored at node $v$ and will be used in the next global update (line $13$).

If the global model is updated at node $v$, \ie if $s_1^{v_0}=1$, then node $v$ creates a $message$ and sends it to one of its neighbors if (i) $visited \neq \mathcal{V}$: when not all nodes are visited in the current synchronization round; or (ii) $\mod(t,H)=0$: this is an indicator of the start of a new synchronization round, which happens periodically at every $H$ iterations.
In other words, global model synchronization continues until all nodes in $\mathcal{V}$ are visited. Then, global model update is paused until a new synchronization round (satisfied by line $17$), which starts at every $H$ iteration. We will describe how $H$ should be selected later in the paper as part of our convergence analysis and evaluations.
If one of the conditions in line $14$ is satisfied, then node $v$ sends the global model to one of its neighbors by calling Alg. \ref{alg:traverse}.

\textbf{Sending Global Model.} Alg. \ref{alg:traverse} describes the logic of \ours at node $v$ for sending a global model to a neighboring node. 
%
%
Alg. \ref{alg:traverse} implements a Depth-First Search (DFS) to traverse all the nodes in the network in a synchronization round. 
%
If $v$ is not visited before in this synchronization round, it is added to $visited$ (line $3$) and its parent node $p^v_m$ is set to $pre\_node$ (line $4$). 
%
The parent node is the node that node $v$ receives the global model from for the first time in this synchronization round. 
%
$\mathcal{C}$ is a set of nodes that node $v$ can possibly transmit.
It includes all of the neighboring nodes which are not in the $visited$ set. 
If $\mathcal{C}$ is not empty, one of its elements $v'$ is chosen randomly (line $8$), and a $message$ including the global model is transmitted to node $v'$ from node $v$.  
%
If $\mathcal{C}$ is empty, \ie all the neighbors of node $v$ are visited in the current synchronization round, the $message$ is sent to the parent of node $v$ ($p^v_m$) (line $11$). We note that if all the nodes are visited in the network, Alg. \ref{alg:Dec} pauses global model update (line $17$), and Alg. \ref{alg:traverse} is not called. Note that Alg. \ref{alg:traverse} and Alg. \ref{alg:Dec} operate simultaneously; one does not need to stop and wait for the other as also illustrated in Fig. \ref{digest}.

\subsection{Multi-Stream \ours}

\subsubsection{Tree Construction and Multiple Streams}


{Our goal in multi-stream \ours is to find a shortest-path tree so that model updates can be distributed quickly. We use a classical distance vector routing algorithm such as Bellman-Ford to construct a shortest-path tree. Belmann-Ford algorithm is optimal \cite{Bellman1958ONAR} in the sense that it results a shortest path tree when a root node is fixed. Multi-stream \ours finds the best tree among all shortest path trees (constructed for each node using Bellman-Ford). The best tree is the tree that has the smallest range (distance from the root to the farthest leaf node). In particular, our  first step is to create a rooted tree from our undirected graph $G$ via Bellman-Ford, where each node $v$ in graph $G$ learns its delay distance $d_{vu}^G$ to node $u$ in a decentralized manner and via message passing.
%
We define  the radius of node $v$ as $R^G_v$, which is the largest distance from node $v$; i.e., $R^G_v = \max_u \{d^G_{vu}\}$. The root of the network is the node with the smallest $R^G_v$, i.e., $r = \arg \min_v \{R^G_v\}$, where $r$ is the root node. The shortest delay tree $ST_r$ rooted with $r$ is constructed in a decentralized manner as each node keeps $d_{vu}^G$ information.} 

{After the tree is constructed, multiple streams are created to exchange the global model in the network. First, the root node creates a number of streams which is equal to the number of its children. 
Each of these streams has a range, which starts from the root node and ends at a child node if the child node itself has more than one child. In that case, the child node behaves exactly as a root node and creates multiple streams towards its children by following the same rule that we just described for the root node. 
Eventually, there will be $M$ streams in the tree, and the set of the streams that go through node $v$ is $\mathcal{M}_v$.}
The set of nodes that are in the range of stream $m$ is $\mathcal{V}_m$.
The set of the streams between root $r$ and node $v$ is defined as $P^r_v$.

  \begin{algorithm}[t!]
\caption{\ours on node $v \in \mathcal{V}$ with multiple streams.}\label{alg:Dec_multi_st}
\begin{algorithmic}[1]
\State \textbf{Initialization: }$\xv[0]=\xz$, $\xv[-1]=\xz$, $queue = ()$.
\For{$m$ in $\mathcal{M}_v$}
\State $\xtild[0][m]=\xz$, $\xtild[-1][m]=\xz$, $visited[m]=\{\}$, $pre\_node[m] = v$, $\mathcal{S}^v_T[m] = \{0\}_{0<t\leq T}$, $s_1^{v_m}[m]=1$.
\EndFor
\For{$t$ in $0,...,T-1$}
\State Sample $i^v_t$ uniformly from $\data{v}$.
\State Start computing the gradient $\nabla \fiv(\xv)$.
\State $\xv[t+1] = \xv- \sum_{z\in \uv} \eta_{z} \nabla \fiv[z](\xv[z])$ 
\If{$queue \neq ()$}
    \For {any $message$ in $queue$}
    \State $(\xtild[t][m], visited[m], pre\_node[m],m)\hspace{-6pt}\leftarrow message$
    \State  $s^v_{t+1}[m]=1$
    \State Remove $message$ from $queue$
    \EndFor
\EndIf
\For {$m$ in $\mathcal{M}_v$}
\If{$s^v_{t+1}[m]=1$} 
    \State $\xtild[t+1][m] = \xtild[t][m] + \frac{D_v}{D} (\xv[t+1]-\xv[-1]) + (\xv[-1] - \xtild[-1][m])$
    \State $\xv[t+1] =\xtild[t+1][m]$ 
    \State$\xv[-1] = \xv[t+1] $ \Comment{Last updated model at  $v$}
    \State $\xtild[-1][m]=\xtild[t+1][m]$ \Comment{Last updated model at node $v$ corresponding to stream $m$}
    \If{$\mod(t,H_m)=0$ or $visited[m] \neq \mathcal{V}_m$}
    \State Send $message$ $=$ $\xtild[t+1][m],$ $visited[m],$   $pre\_node[m],$ $r$ to a neighboring node by calling Alg.\ref{alg:traverse}. 
    %
    \Else
    \State $s^v_{t+H-mod(t,H_m)}[m]=1$
    \State $visited[m]=\{\}$
    \EndIf
\EndIf
\EndFor
\EndFor
\end{algorithmic}
\end{algorithm}

\vspace{-15pt}
\subsubsection{Algorithm Design} 

Multi-stream \ours is  summarized in Alg. \ref{alg:Dec_multi_st}. The following are the main differences between Algs. \ref{alg:Dec_multi_st} and \ref{alg:Dec}.

There are multiple global models in different streams, \ie $\xtild[t][m]$ corresponds to the global model in stream $m$ out of $M$ streams.
There are $|\mathcal{M}_v|$ models stored in each node, \ie $\xtild[-1][m]$ to represent the global model corresponding to the last synchronization of stream $m$ at node $v$. 
We define $visited[m]$, $pre\_node[m]$, and $s_t^{v}[m]$ for each stream $m$.

Each node $v$ has a $queue$  to store all the messages that a node receives from its neighbors. It is initialized as an empty queue at the start. Whenever  node $v$ receives a $message$ from one of its neighbors, it is added to the queue. 
Each node can receive up to $|\mathcal{M}_v|$ messages related to different streams, so the size of the $queue$ is $|\mathcal{M}_v|$. 
In each message, there is a stream index $m$ (line $10$). 

Node $v$ extracts all the messages in its queue (lines 9-12). Then, it updates its global and local models as in Alg. \ref{alg:Dec} if $s^v_{t+1}[m]=1$.
The global model is updated using the most recent local updates of node $v$ and global updates of other streams (line $15$). The global model synchronization continues until all nodes in $\mathcal{V}_m$ are visited for stream $m$. Then, global model update is paused until a new synchronization round, which starts at every $H_m$ iteration. The policy for selecting $H_m$ is explained in the next section.


\section{Convergence Analysis of \ours}
We use the following assumptions for the convergence analysis of single- and multi-stream \ours. 

\begin{enumerate}[leftmargin=*]
    \item \textbf{Smooth local loss.} \label{as1}
$f^v$ is continuously differentiable and its gradient is $L$-Lipschitz for $1 \leq v \leq V$, \ie $\normm{\nabla f^v(\vec{y})-\nabla f^v(\vec{x})} \leq L \normm{\vec{y}-\vec{x}}, \quad \forall \vec{x},\vec{y} \in \mathbb{R}^d$.
%

%
\item \textbf{Bounded local variance.} \label{as3}
The variance of the stochastic gradient is bounded for all nodes, \ie $0 \leq t < T$, $1 \leq v \leq V$, $ \EX_{i^v_t} \norm{\nabla \fiv(\xv) - \nabla f^v(\xv)}\leq \sigma^2$.

\item \textbf{Bounded diversity.} \label{as4}
The diversity of the local loss functions and global loss function is bounded, \ie $0 \leq t < T$, $1 \leq v \leq V$, $ \norm{\nabla f^v(\xv) - \nabla f(\xv)}\leq \zeta^2$.
\item \textbf{Bounded lag.}\label{as5}
We assume bounded lag, \ie $\max\{\lv-t\}\leq E, \quad 0 \leq t < T, 1 \leq v \leq V$. 
\item \textbf{Bounded synchronization interval.}\label{as6}
For single-stream \ours, we assume that the interval between two subsequent global model synchronizations is bounded, \ie $gap(\mathcal{S}^v_T) \leq H$, $1 \leq v \leq V$,
where $gap(\mathcal{S}^v_T)$ shows the maximum gap between two subsequent $1$s in  $\mathcal{S}^v_T$. 
For multi-stream \ours, we assume different bounds 
 for each stream, i.e., $gap(\mathcal{S}^v_T[m]) \leq H_m$ for  $m \in \mathcal{M}_v$. 
\item \textbf{Convexity.} \label{as2}
$f$ is $\mu$-(strongly) convex, \ie $\forall \vec{x},\vec{y} \in \mathbb{R}^d$, $ f(\vec{y}) \geq f(\vec{x}) + \inpr{\nabla f(\vec{x})}{\vec{y}-\vec{x}}+ \frac{\mu}{2} \norm{\vec{y}-\vec{x}}$.
\end{enumerate}

\begin{theorem}\label{T1}
Let assumptions \ref{as1}-\ref{as6} hold, with a constant and small enough learning rate $\eta \leq \frac{1}{30LA}$ (potentially depending on $T$), the convergence rate of single- and multi-stream \ours is as follows:

\vspace{5pt}
\textbf{Non-convex:} $\frac{1}{T}\sum_{t=0}^{T-1}\EX\norm{\nabla f(\xhat)}$ is
\begin{align}
       & O \bigg(\frac{FLA}{T} + \sigma\sqrt{\frac{\rho LF}{T}} + (\frac{LF\sqrt{\sigma^2A + \zeta^2A^2}}{T})^{\frac{2}{3}} \bigg),\nonumber
\end{align}
where $\xhat = \sumov \sum_{t=0}^{T-1} \DvtoD \xv$.

\vspace{5pt}
\textbf{Convex:} Under assumption \ref{as2} for $\mu \geq 0$, $\EX f(\xhat) - f^*$ is
\begin{align}
     & O \bigg(\frac{RLA}{T} + \sigma\sqrt{\frac{\rho R}{T}} + (\frac{R\sqrt{L(\sigma^2A + \zeta^2A^2)}}{T})^{\frac{2}{3}} \bigg),\nonumber
\end{align}
where $\xhat = \sumov \sum_{t=0}^{T-1} \DvtoD \xv$.

\vspace{5pt}
\textbf{Strongly-convex:} Under assumption \ref{as2} for $\mu > 0$, $\EX f(\xhat) - f^*$ is
\begin{align}
     \Tilde{O} \bigg(RLA \exp(\frac{-\mu T}{LA}) + \frac{\rho \sigma^2}{\mu T}+\frac{L(\sigma^2A + \zeta^2A^2)}{\mu^2 T^2}\bigg),\nonumber
\end{align}

where $\xhat = \frac{1}{DW_T}\sumov \sum_{t=0}^{T-1} D_v\omega_t \xv$, $\omega_t = (1-\mu\eta)^{-(t+1)}$, $W_T = \sum_{t=0}^{T-1} \omega_t$. 

\vspace{5pt}
$\Tilde{O}$ hides constants and poly-logarithmic factors, $T$ represents the wall clock time, $F := f(\xz)-f^*$, $R:=\norm{\xz-\xstar}$, $A:=H'+E$, and $\rho := \sumov (\DvtoD)^2$. The convergence rate of single-stream \ours follows when $H'=H$, and the convergence rate of multi-stream \ours is obtained by putting $H'=\max_{v}\sum_{m\in P_v^r}H_m$ in A. \hfill $\Box$

\end{theorem}

\begin{table}[b]
\caption{Constraints on $H$ to get linear speed-up.}
\label{Tab}
    \centering
    \begin{tabular}{llll}
        \toprule
        \multirow{2}{*}{} & \multirow{2}{*}{Converge rate} & \multicolumn{2}{c}{{$H'+E$}}\\
         \cmidrule{3-4} \\
         \addlinespace[-12pt]
        {} & {} &  iid & non-iid \\
        \midrule
        Strongly convex & $\Tilde{O}(\frac{\rho}{T})$  & $\Tilde{O}(\rho T)$ & $\Tilde{O}(\rho^\frac{1}{2} T^\frac{1}{2}$)  \\
        Convex & $O(\sqrt{\frac{\rho}{T}})$  & $O(\rho^{\frac{3}{2}} T^{\frac{1}{2}})$ & $O(\rho^{\frac{3}{4}} T^{\frac{1}{4}})$  \\
        Non-convex & $O(\sqrt{\frac{\rho}{T}})$  & $O(\rho^{\frac{3}{2}} T^{\frac{1}{2}})$ & $O(\rho^{\frac{3}{4}} T^{\frac{1}{4}})$  \\
        \bottomrule
    \end{tabular}
\end{table}

{\begin{corollary} \label{remark1}
For strongly-convex case, in iid data distribution over nodes, i.e., $\zeta=0$, the convergence rate to the optimum value $f^*$ is $\Tilde{O}(\frac{\rho}{T})$ given that $H'+E= \Tilde{O}(\rho T)$ is satisfied, where $\rho = \sumov (\DvtoD)^2$ is a data concentration coefficient that can take values between $\frac{1}{V}\leq \rho < 1$. In non-iid data distribution over nodes ($\zeta \neq 0$), linear speed up $O(\sqrt{\frac{\rho}{T}})$ is achieved when $H'+E= \Tilde{O}(\sqrt{\rho T})$ holds. This restriction in other scenarios is depicted in Table \ref{Tab}. \hfill $\Box$
\end{corollary}}

Theorem \ref{T1} and Corollary \ref{remark1} show a nice trade-off between convergence rate and  communication overhead. It determines how much communication is needed to achieve a linear speed-up. Remark \ref{remark1} also shows the impact of non-iid data distribution, which requires smaller $H'$, hence more communications to converge and achieve linear speed-up.


\begin{corollary}
Corollary \ref{remark1} shows that the linear speed up is achieved when $T = \Tilde{\Omega}(\frac{H'}{\rho})$  and $T = \Tilde{\Omega}(\frac{H'^2}{\rho})$ for iid and non-iid data, respectively. When the network is larger, single-stream \ours needs longer $H'$ (which is equal to $H$) to visit all the nodes, which requires larger $T$ (convergence time). But in multi-stream \ours, $H'$ defined as $H'=\max_{v}\sum_{m\in P_v^r}H_m$ could be as low as $R_r^G$, which is the radius of root node $r$ (or maximum delay toward any node from root node $r$). As $R_r^G$ does not necessarily increase with the size of the network, multi-stream \ours is plausible even for large networks. 
\end{corollary}

\begin{corollary}
Let's assume that the network can be covered in $H$ iteration using single/multi-stream approach. \ours can efficiently perform synchronization while nodes are doing local-SGD, i.e., network topology, spectral gap, or the maximum and minimum degrees in the network topology do not affect the convergence rate. This is one advantage of using \ours in comparison to previous works on asynchronous decentralized learning like \cite{swarm} where the convergence rate in non-convex setting is $O \bigg(\frac{F}{\sqrt{HVT}} + \frac{\sqrt{H}(\sigma^2+H\zeta^2)}{\sqrt{VT}} + \frac{Vd_{max}L^2H^3G^2}{d_{min}\lambda^2T} \bigg)$. Here, we observe that the minimum degree ($d_{\text{min}}$), maximum degree ($d_{\text{max}}$), and spectral gap ($\lambda$) of the network graph are part of the result, so affects the convergence.
\end{corollary}




\textbf{Sketch of Proof of Theorem \ref{T1}.} (The details of the proof is provided in the Appendix \ref{appendixb}.)  We define a virtual sequence $\{\xbar\}_{t\geq 0}$ as $\xbar =\xz - \sumov \sum_{z=0}^{t-1} \frac{D_v}{D} \eta_z \nabla \fiv[z](\xv[z])$ following a similar idea in \cite{Stich2019LSGD}.
Lemma \ref{lem1-main}, and \ref{lem2-main} indicates how the convergence criteria are related to $\EX\norm{\xbar - \xv}$, the deviation between virtual and actual sequences and we find an upper-bound for this term in Lemma \ref{lem3-main}. 
%
%

\begin{lemma}\label{lem1-main}
If assumptions \ref{as1}-\ref{as3}, \ref{as2} for $\mu > 0$ hold, and $\eta_t \leq \frac{1}{4L}$, then $\EX\norm{\xbar[t+1]  - \xstar} \leq (1-\mu\eta_t)\EX\norm{\xbar - \xstar} + \eta_t^2\rho\sigma^2 -\eta_t \EX(f(\xbar)-f(\xstar))+ 2L\eta_t\sumov \DvtoD \EX \norm{\xv-\xbar}$.
\end{lemma}

\begin{lemma}\label{lem2-main}
If assumptions \ref{as1}-\ref{as3} hold, and $\eta_t \leq \frac{1}{4L}$, then
$\EX f(\xbar[{t+1}]) \leq \EX f(\xbar) + \frac{\eta_t^2L\rho\sigma^2}{2} -\frac{\eta_t}{4} \EX\norm{\nabla f(\xbar)}+ L^2\eta_t\sumov \DvtoD \EX \norm{\xv-\xbar}$.
\end{lemma}

\begin{lemma}[Bounding deviation]\label{lem3-main}
If assumptions \ref{as3}-\ref{as6} hold, $\eta_t=\eta \leq \frac{1}{30L(H'+E)}$, and $\omega_{t}$ is $(H'+E)$-slow increasing then $\sum_{t=0}^{T-1} \omega_t \sumov \DvtoD \EX \norm{\xbar - \xv} \leq \frac{1}{8L^2}\sum_{t=0}^{T-1} \omega_t  \EX\norm{\nabla f(\xbar)}+90\eta^2 (H'+E) \big( \zeta^2 (H'+E) + \sigma^2\big)  \sum_{t=0}^{T-1}  \omega_t$,
where in single-stream $H'=H$ and in multi-stream \ours $H'=\max_{u}\sum_{m\in p_u^v}H_m$.
\end{lemma}

\section{Evaluation of \ours} \label{eval}

\begin{figure*}[t]
     \centering
     \begin{subfigure}[b]{0.325\textwidth}
         \centering
         \includegraphics[width=\textwidth]{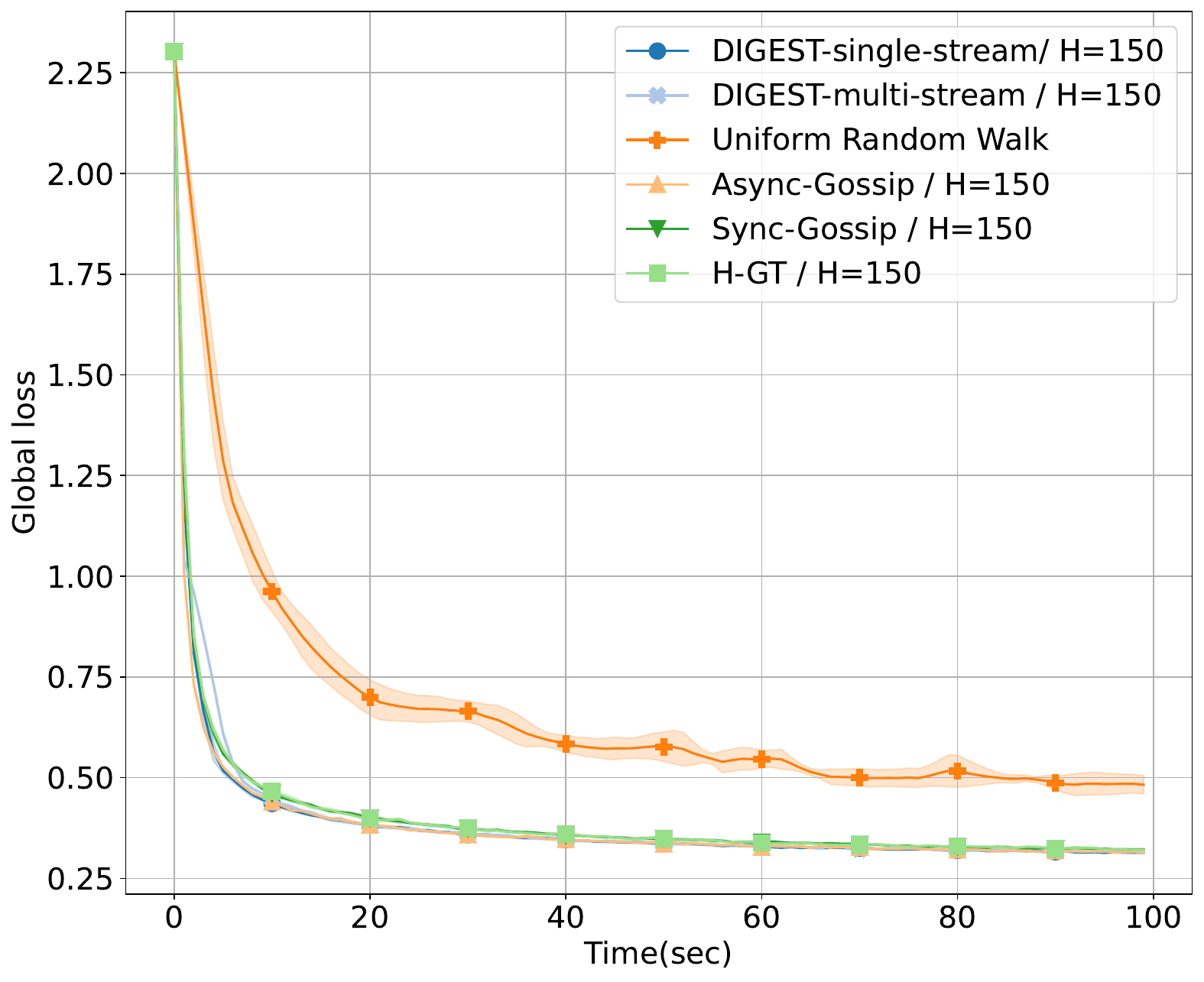}
         \caption{10-nodes / iid / balanced}
         \label{mnist/final_l_t_10_iid}
     \end{subfigure}
     \begin{subfigure}[b]{0.325\textwidth}
         \centering
         \includegraphics[width=\textwidth]{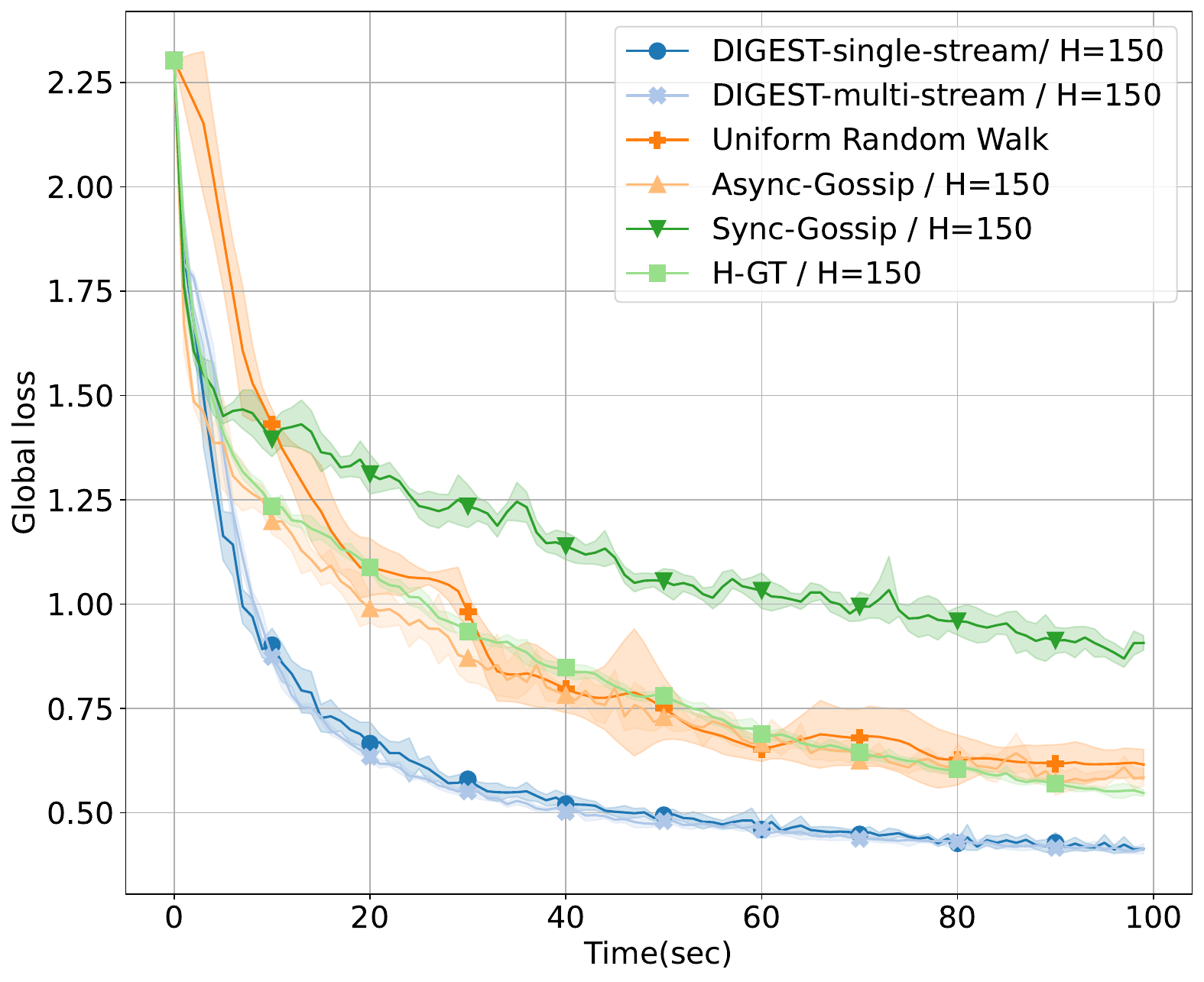}
         \caption{10-nodes / non-iid / unbalanced}
         \label{mnist/final_l_t_10_noniid}
     \end{subfigure}
     \begin{subfigure}[b]{0.325\textwidth}
         \centering
         \includegraphics[width=\textwidth]{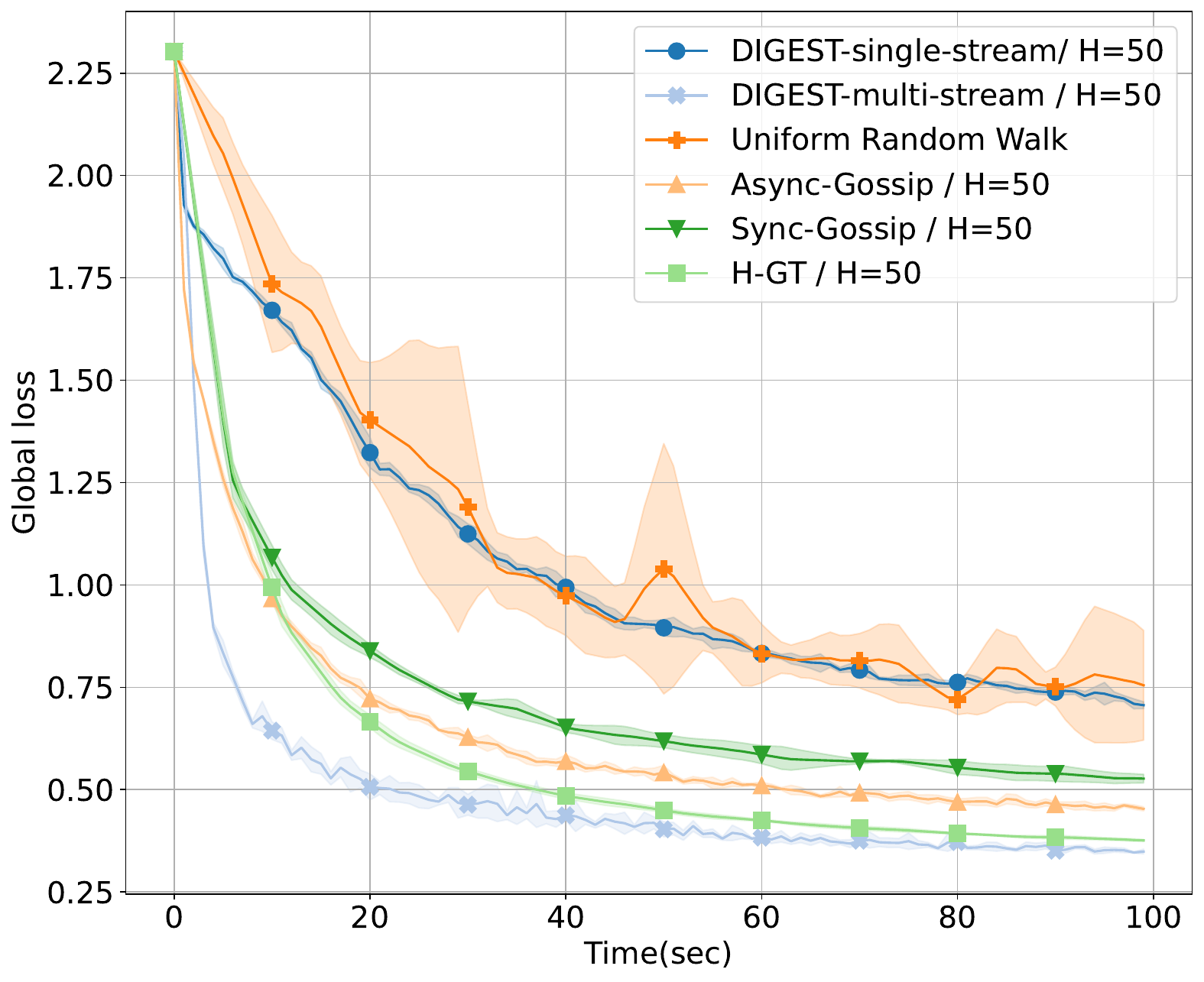}
         \caption{100-nodes / non-iid / unbalanced}
         \label{mnist/final_l_t_100_noniid}
     \end{subfigure}
     \vspace{-10pt}
        \caption{{Convergence results for \textit{MNIST} dataset in terms of global loss over wall-clock time.}}
        \label{mnist-10-100}
        \vspace{-5pt}
\end{figure*}

\begin{figure*}[t]
     \centering
     \begin{subfigure}[b]{0.325\textwidth}
         \centering
         \includegraphics[width=\textwidth]{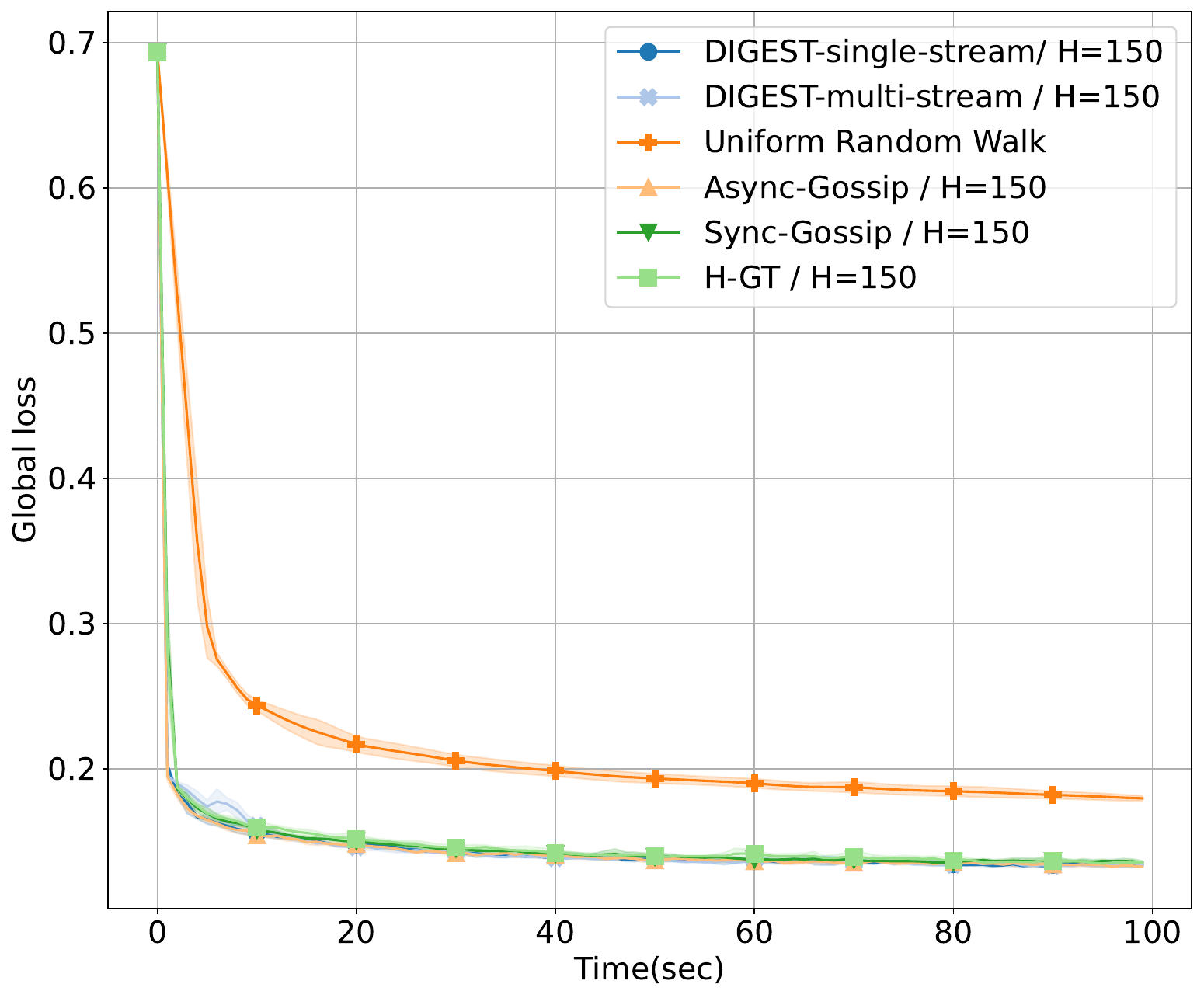}
         \caption{10-nodes / iid / balanced}
         \label{w8a/final_l_t_10_iid}
     \end{subfigure}
     \begin{subfigure}[b]{0.325\textwidth}
         \centering
         \includegraphics[width=\textwidth]{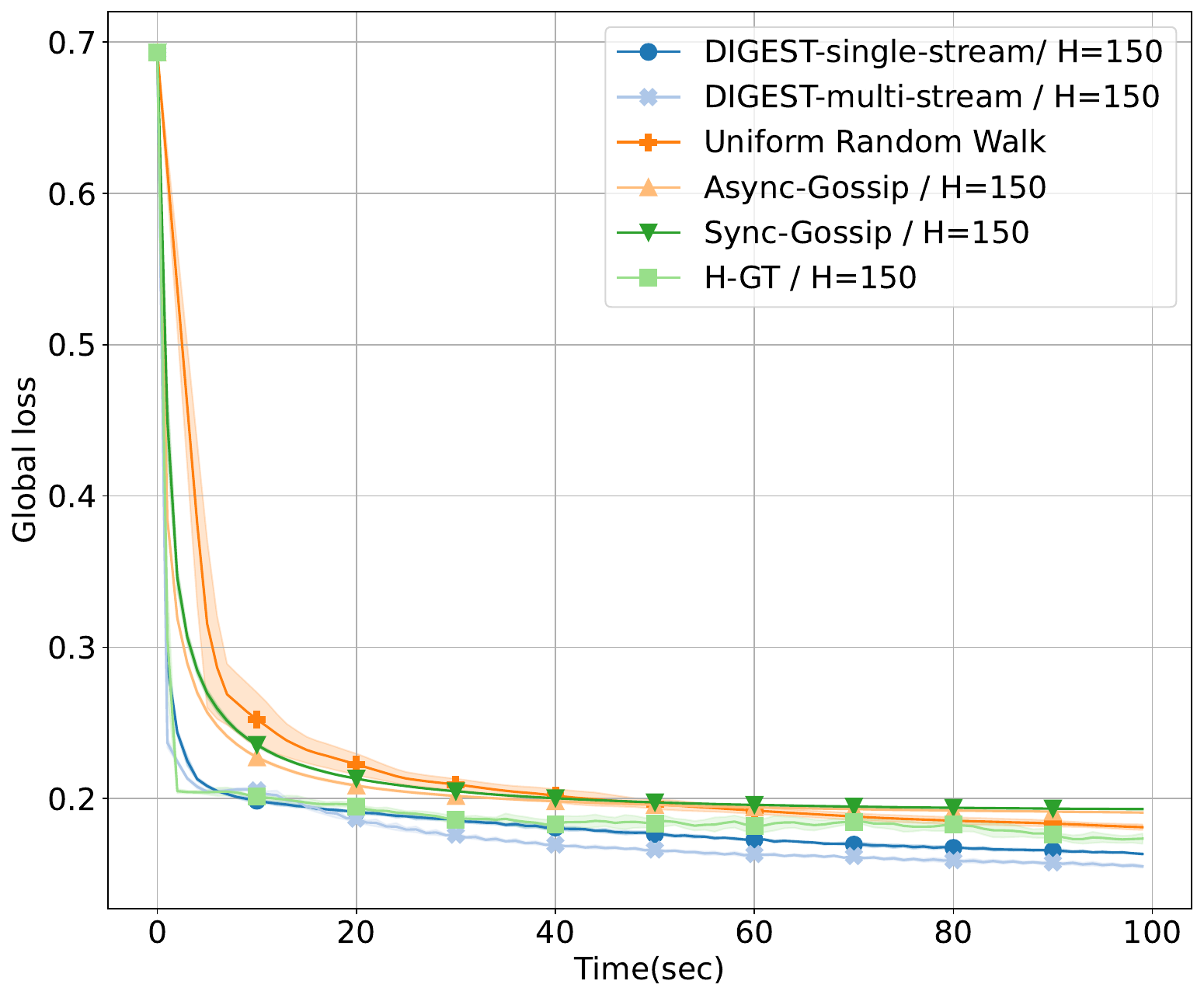}
         \caption{10-nodes / non-iid / unbalanced}
         \label{w8a/final_l_t_10_noniid}
     \end{subfigure}
     \begin{subfigure}[b]{0.325\textwidth}
         \centering
         \includegraphics[width=\textwidth]{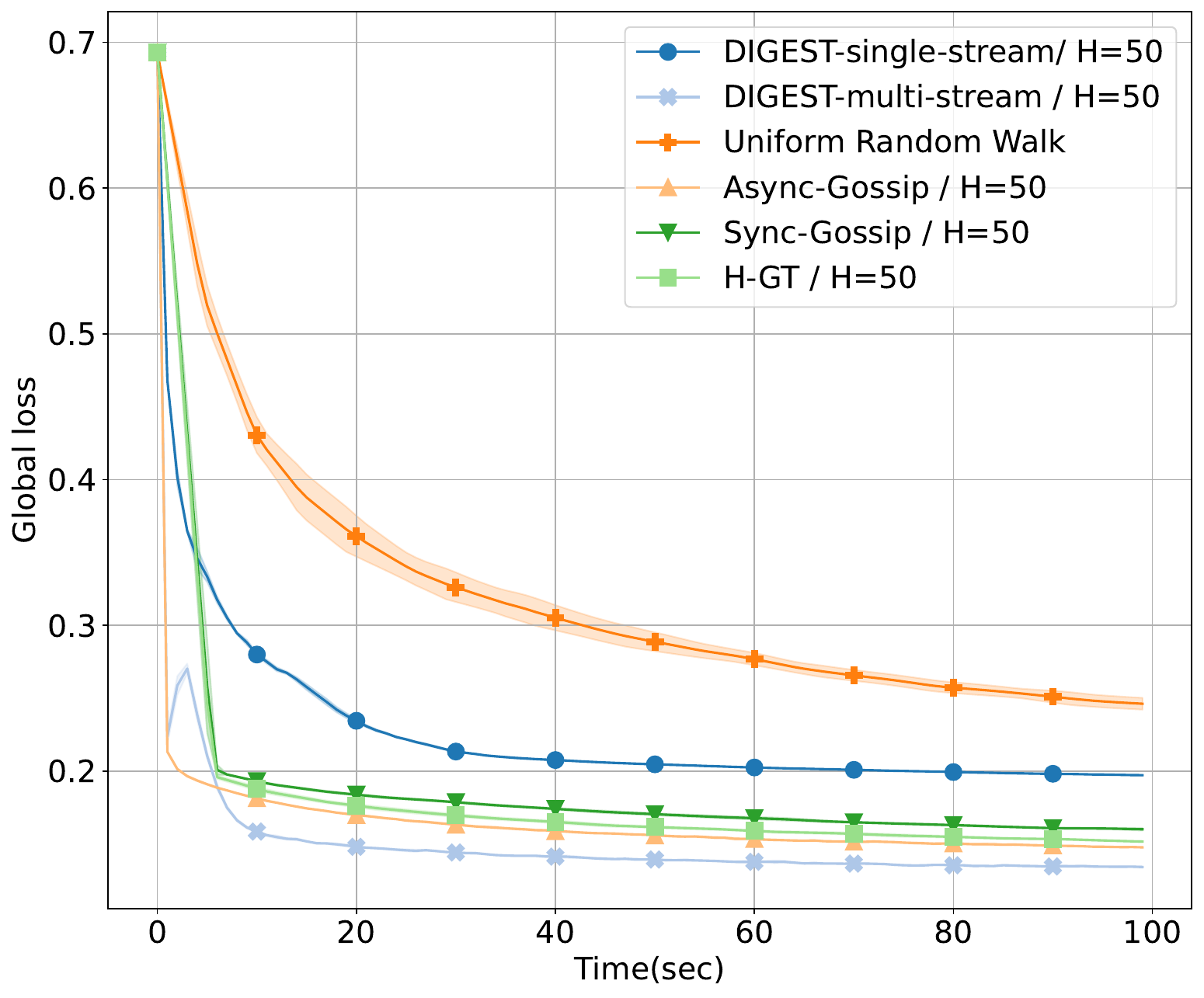}
         \caption{100-nodes / non-iid / unbalanced}
         \label{w8a/final_l_t_100_noniid}
     \end{subfigure}
     \vspace{-10pt}
        \caption{{Convergence results for \textit{w8a} dataset in terms of global loss over wall-clock time.}}
        \label{w8a-10-100}
        \vspace{-15pt}
\end{figure*}

We evaluate \ours as compared to baselines; 
(i) Uniform Random-Walk (URW) \cite{GhadirRW2021}: This is a random walk-based learning algorithm described in Fig. \ref{RW}; 
(ii) Gradient tracking (GT) with local-SGD \cite{liu2023GT}: It is an algorithm that is developed to overcome data heterogeneity across nodes in a decentralized optimization problem;
(iii) Async-\gos \cite{Lian2018Async} with local-SGD; (iv) (ii) Sync-\gos \cite{Lian2018Async} with local-SGD. Our codes are provided in \cite{digest-codes}.

 We consider two network topologies; an Erd\H{o}s–Rényi graph of $V = 10$ and $V= 100$ nodes with $0.3$ as the probability of connectivity.
{Each neighboring pair's communication delay is assumed to conform to an exponential distribution. The average delay is randomly chosen to span from 0 to 50 times the duration of the specific model's local SGD computation.}

We use two data distributions: (i) iid-balanced, and (ii) non-iid-unbalanced. In iid-balanced case, data is shuffled and equally divided and placed in nodes. Non-iid-unbalanced has two features: (i) Non-iid, which is realized by sorting data according to their labels and distributing them in the sorted order. Thus, the data distributed over nodes will be non-iid; (ii) Unbalanced, which means that each node may have a different amount of data. We use geometric series to realize unbalanced data across nodes. For example, if a node $u$ has $D_u = \delta$ data, the next nodes get $\delta \rho$, $\delta \rho^2$, etc. data, where $\rho$ is determined by taking into account the size of the total dataset $D$. 

\subsection{Logistic Regression}

We examine the convergence performance of logistic regression, \ie $ f(\vec{x})$ $=$ $\frac{1}{D}$ $ \sum_{i=1}^D \text{CrossEntropy}$ $\big(\text{softmax}(\vec{x}\vec{a_i}),b_i\big)+\frac{\lambda}{2}\norm{\vec{x}}$, 
where $\vec{a_i}\in \mathbb{R}^d$, and $b_i$ are the feature and label of the data sample $i$. The regularization parameter is considered $\lambda = \frac{1}{D}$.
We run the optimization using a tuned constant learning rate for each algorithm.
To grid-search the learning rate, we try the experiment by multiplying and dividing the learning rate by powers of two. This involves trying out both larger and smaller learning rates until we find the best result. 
We use datasets \textit{w8a} \cite{w8a} and \textit{MNIST} \cite{mnist}.
The numerical experiments were run on Ubuntu $20.04$  using $36$ Intel Core i9-10980XE processors. 
For each experiment, we repeat $50$ times and present the error bars associated with the randomness of the optimization.
In every figure, we include the average and standard deviation error bar.

Figs. \ref{mnist-10-100} and \ref{w8a-10-100} show the convergence behavior of our algorithms as well as the baselines for \textit{MNIST} and \textit{w8a} datasets in 10-nodes and 100-nodes topologies.
URW generally underperforms as compared to other methods due to its approach of conducting only one local-SGD operation per iteration on a single node. As a consequence, it does not have a linear speed-up with increasing number of nodes.
In certain situations involving non-iid data distribution, URW may exhibit better performance than some other methods as shown in Figs. \ref{mnist/final_l_t_10_noniid}, \ref{mnist/final_l_t_100_noniid}, \ref{w8a/final_l_t_100_noniid}, \ref{w8a/final_l_t_100_noniid}. This is because URW is not affected by non-iidness as it uniformly incorporates data from all nodes. \ours, Sync-\gos, and Async-\gos have similar performance in iid data distribution in Figs. \ref{mnist/final_l_t_10_iid}, \ref{w8a/final_l_t_10_iid}. On the other hand,  we observe that \gos based algorithms are suffering from slow convergence in non-iid setting as shown in Figs.
\ref{mnist/final_l_t_10_noniid}, \ref{mnist/final_l_t_100_noniid},
\ref{w8a/final_l_t_10_noniid}, \ref{w8a/final_l_t_100_noniid}. 
We also observe that GT algorithms enhance the performance of gossip-based algorithms by incorporating a mechanism to overcome non-iidness. 
However, this algorithm demands twice the communication overhead compared to sync-\gos, resulting in more communication overhead, which can degrade its convergence performance in terms of wall-clock time.
In comparison, \ours has better convergence behavior thanks to its very design of spreading information uniformly in the network to handle non-iidness. 
%
It is evident that when the network is larger, one-stream \ours method is unable to cover the entire network as quickly as required, highlighting the need to utilize multi-stream \ours to overcome this limitation. This observation is supported in Figs. \ref{mnist/final_l_t_100_noniid}, \ref{w8a/final_l_t_100_noniid}, where all streams have the same $H_m = H, m\in \mathcal{M}$ in multi-stream \ours. 

\subsection{Deep Neural Network (DNN)} In this section, we use ResNet-$20$ \cite{He2015resnet} as the DNN model. The dataset is  \textit{CIFAR-10} \cite{cifar10}. 
We have set the batch size to $36$ per node, and the learning rate is decayed by a constant factor after completing $50\%$ and $75\%$ of the training time. The initial value of the learning rate is separately tuned for each algorithm.
We have set the momentum value to $0.9$ and the weight decay to $10^{-4}$.
We observe that in the iid setting (Fig. \ref{cifar10/iid}), all algorithms except URW perform similarly. However, in non-iid settings, where communication and model distribution across the network become crucial, \ours outperforms \gos-based algorithms, Fig. \ref{cifar10/noniid}. 


\begin{figure}[t]
\vspace{-5pt}
     \centering
     \begin{subfigure}[b]{0.41\textwidth}
         \centering
         \includegraphics[width=\textwidth]{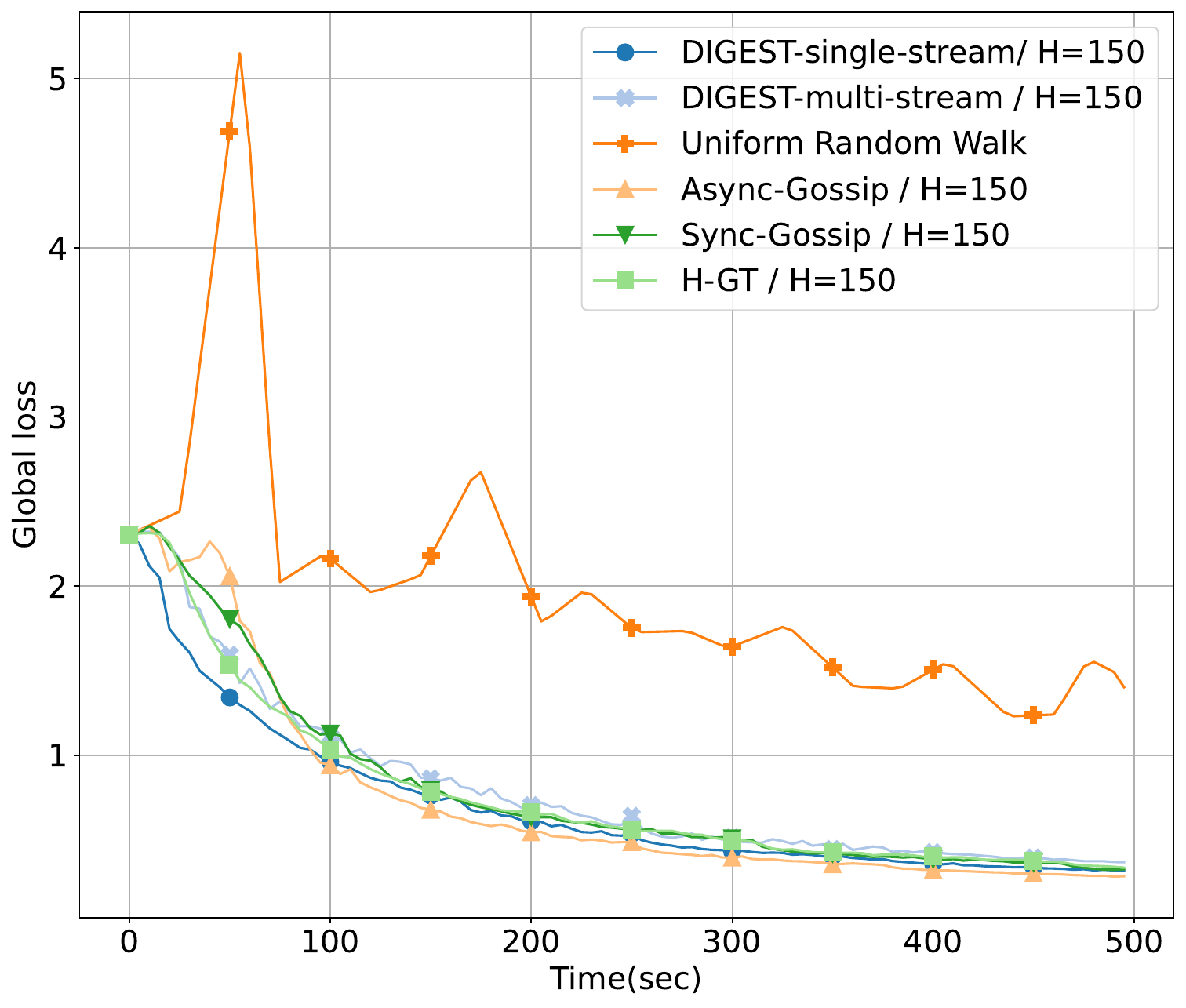}
         \caption{10-nodes / iid / balanced}
         \label{cifar10/iid}
     \end{subfigure}
     \hfill
     \begin{subfigure}[b]{0.41\textwidth}
         \centering
         \includegraphics[width=\textwidth]{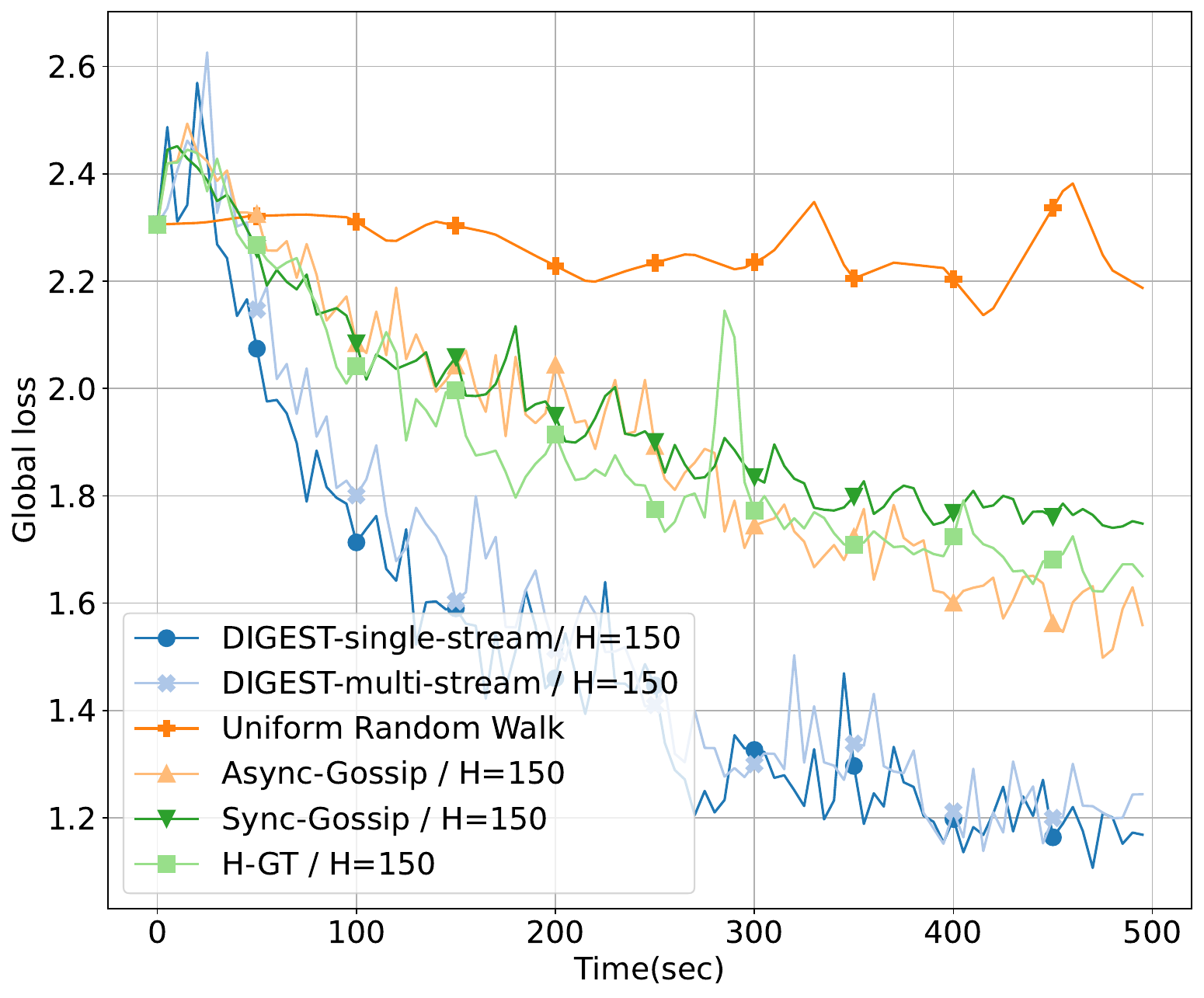}
         \caption{10-nodes / non-iid / unbalanced}
         \label{cifar10/noniid}
     \end{subfigure}
        \caption{{Convergence results for \textit{CIFAR-10} dataset in terms of global loss over wall-clock time.}}
        \label{cifar10}
        \vspace{-10pt}
\end{figure}


\subsection{Speed-up} In this section, we evaluate the speed up performance of our \ours algorithms as well as the baseline; centralized parallel SGD. We consider the following cost function
\begin{align} \label{eq:cost-speed-up}
    f(x) = \begin{cases}
       (x-1)^2 &\quad x\geq 1,\\
       \frac{(x-1)^2}{2} &\quad x< 1. \\ 
     \end{cases}
\end{align}
{to control the non-iidness and local variances. Note that we need to increase the number of nodes to generate speed-up curves, so we need to create a non-iid data distribution over the nodes. Creating a uniform non-iid distribution using real datasets when the number of nodes increases is very difficult. Thus, we use a pre-defined cost function in (\ref{eq:cost-speed-up}) to verify our theoretical results following a similar approach in \cite{spiridonoff2021communication}.}
In particular, we employ Local-SGD at node $v$ with gradients affected by a normal noise, i.e., $\nabla \fiv(\xv) = \nabla f(\xv) + n_t^v$, where $n_t^v\sim \mathcal{N}(\zeta_v,\sigma^2)$, $\sum_{v=1}^V \zeta_v = 0$.
To create the speed-up curve, we divide the expected error of a single node SGD by the expected error of each method at the last iteration $T$ for different number of nodes. As in linear speed-up, error decreases linearly with the increasing number of workers, so we expect to see a straight line on the graph. 
The speed-up curve is illustrated in Fig. \ref{speed-up}. The central parallel SGD averages all nodes` updates at every $H$ steps, and updates the model in all nodes. 
%
%
It is worth noting that the central parallel SGD with $H=1$ is the best speed-up that can be achieved in this scenario. 

We set the learning rate to $0.001$, and $|\zeta_v| = 5$ for $v \in \mathcal{V}$, $\sigma =5$, and $T = 10^4$. Note that in iid setting with a less restrictive constraint on $H$, larger $H$ can still lead to linear speed-up when compared to non-iid setting.
Moreover, it is seen that single stream \ours has linear speed-up to a certain limit; however, as the number of nodes increases and single-stream \ours cannot traverse the entire network fast enough, linear speed-up is not maintained. On the other hand, multi-stream \ours achieves linear speed up and achieves a very close performance to the best possible scenario, which is centralized parallel SGD with $H=1$. 

\begin{figure}[t]
     \centering
     \begin{subfigure}[b]{0.41\textwidth}
         \centering         \includegraphics[width=\textwidth]{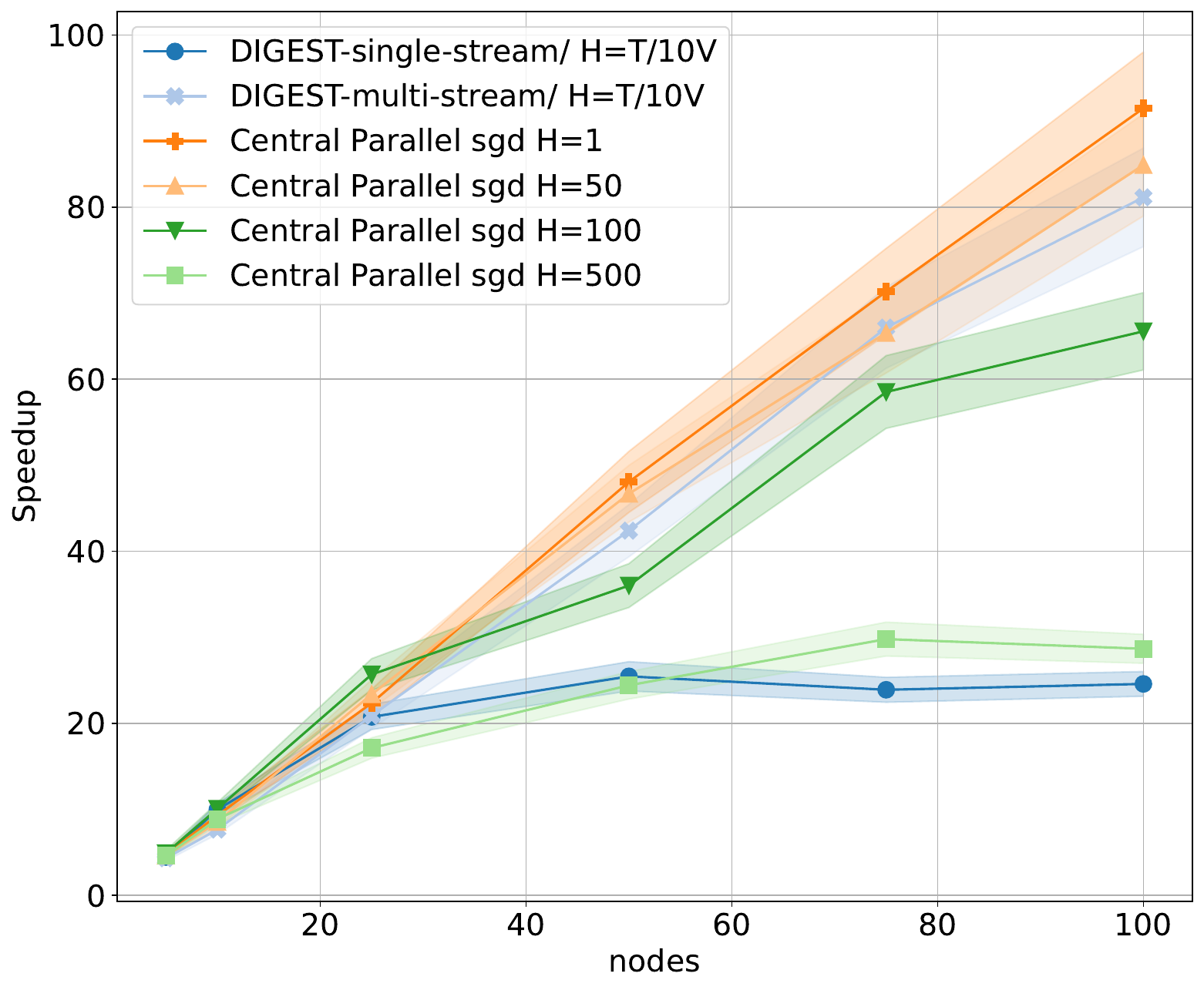}
         \caption{iid / balanced}
         \label{sppeed-up/iid}
     \end{subfigure}
     \hfill
     \begin{subfigure}[b]{0.41\textwidth}
         \centering
         \includegraphics[width=\textwidth]{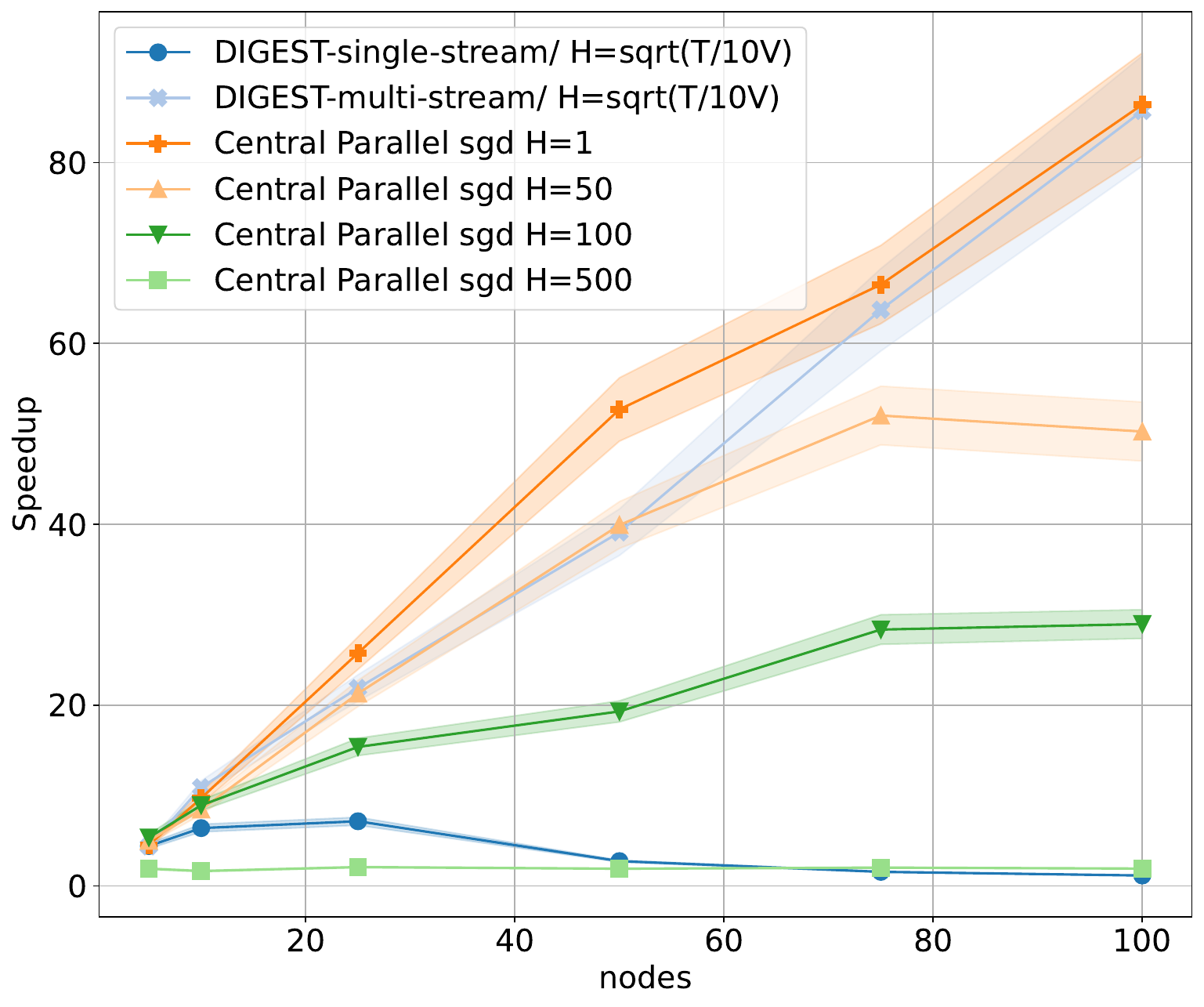}
         \caption{non-iid / unbalanced}
         \label{sppeed-up/noniid}
     \end{subfigure}
        \caption{Speed-up curves for \ours.}
        \label{speed-up}
        \vspace{-10pt}
\end{figure}

\section{Conclusion}

We designed a fast and communication-efficient decentralized learning mechanism; \ours by particularly focusing on stochastic gradient descent (SGD). We designed single- and multi-stream \ours to exploit the convergence rate and communication overhead tradeoff. We analyzed the convergence of single- and multi-stream \ours, and proved that both algorithms approach to the optimal solution asymptotically for both iid and non-iid data. The simulation results confirm that the communication cost of \ours is low as compared to the baselines, and  its convergence rate is better than or comparable to the baselines.

\bibliographystyle{IEEEtran}
\bibliography{main.bib}

\begin{IEEEbiography}[{\includegraphics[width=1in]{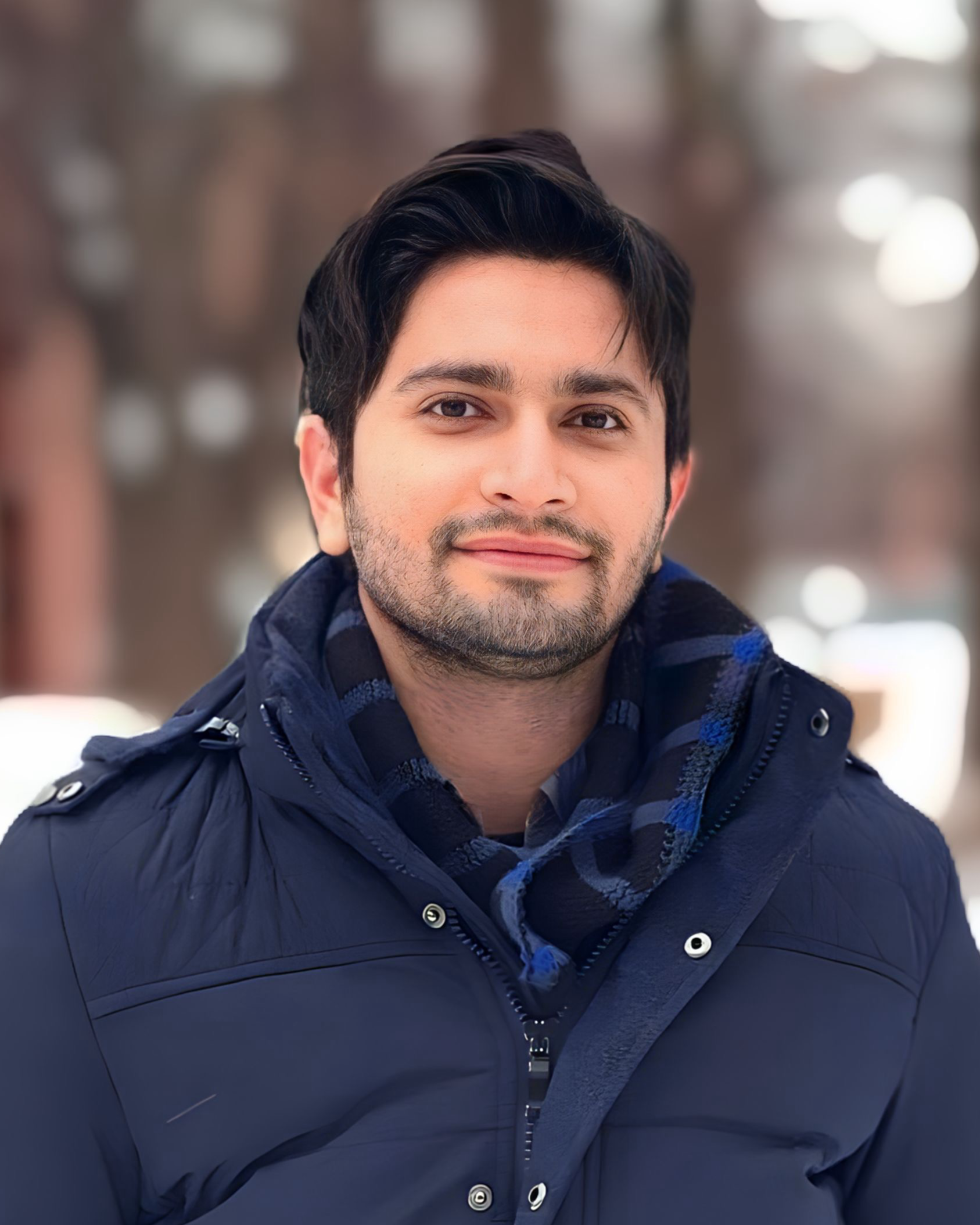}}]
{Peyman Gholami~} received his B.Sc. degree in Electrical Engineering from Iran University of Science and Technology (IUST) in 2018, and M.Sc. degree in Communication Systems from Sharif University of Technology (SUT) in 2021. Currently, he is a Ph.D student in the Department of Electrical and Computer Engineering, University of Illinois at Chicago, under supervision of Prof. Hulya Seferoglu.
\end{IEEEbiography}

\begin{IEEEbiography}[{\includegraphics[width=1in]{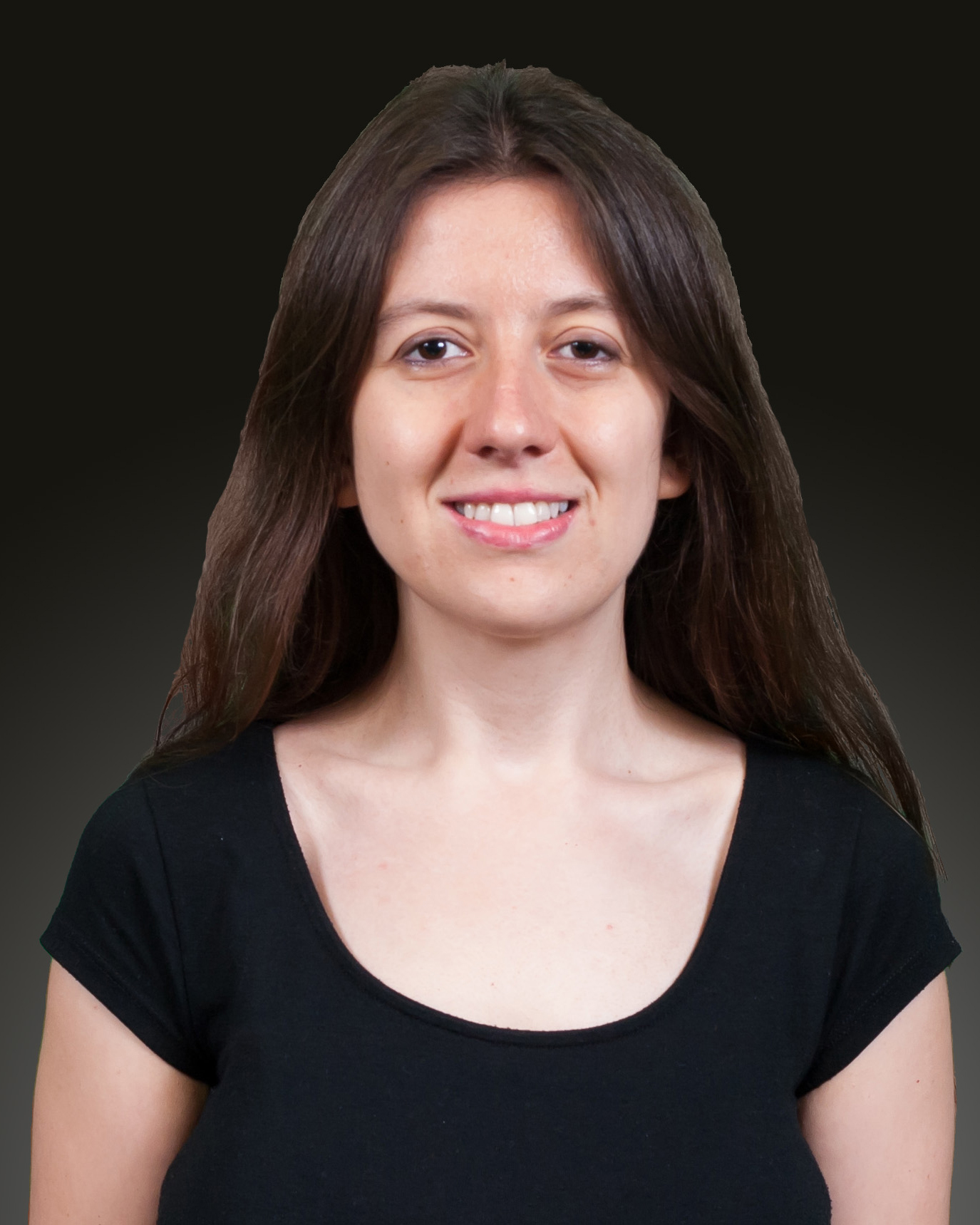}}]
{Hulya Seferoglu~} is an Associate Professor in the Electrical and Computer Engineering Department of University of Illinois at Chicago. 
Before joining University of Illinois at Chicago, she was a Postdoctoral Associate at Massachusetts Institute of Technology. 
She received her Ph.D. degree in Electrical and Computer Engineering from University of California, Irvine, M.S. degree in Electrical Engineering and Computer Science from Sabanci University, and B.S. degree in Electrical Engineering from Istanbul University. 
She has served as an associate editor for IEEE Transactions on Mobile Computing and IEEE/ACM Transactions on Networking. 
She received the NSF CAREER award in 2020. 
\end{IEEEbiography}

\newpage

\vfill\pagebreak
\newpage
\onecolumn

\appendix
\subsection{Notation Table }\label{appendixa}

\begin{tabularx}{0.95\textwidth} { 
  | >{\raggedright\arraybackslash}m{5em}
  | >{\raggedright\arraybackslash}X |}
 \hline
 $G=(\mathcal{V},\mathcal{E})$ & Graph representing the network\\
 \hline
 $V$ & Number of nodes\\
 \hline
 $\mathcal{D}$ & The whole dataset in the network with size $D$\\
 \hline
 $\mathcal{D}_v$ & Subset of $\mathcal{D}$ at node $v$ with size $D_v$\\
 \hline
 $f_i (\vec{x})$ & Loss function of model $\vec{x}$ associated with the data sample $i$ \\
 \hline
 $f(\vec{x})$ & Global loss function of model $\vec{x}$ \\
 \hline
 $f^v(\vec{x})$  & Local loss function of model $\vec{x}$ at node $v$ \\
 \hline
 $f^*$ & $\min_{\vec{x} \in \mathbb{R}^d}f(\vec{x})$\\
\hline
 $\xstar$ & $\arg \min_{\vec{x} \in \mathbb{R}^d}f(\vec{x})$\\
\hline
 $\xz$ & Initial model\\
 \hline
$T$ & Total number of iterations\\
 \hline
 $\eta_t$ & Learning rate at iteration $t$\\
\hline
$\lv$  & Completion time of local-SGD update started at $t$ in node $v$\\
\hline
${L}^v_T$ & Set of $\{\lv\}_{0\leq t <T}$\\
\hline
$\xv$ & Local model in node $v$ at $t$\\
\hline
$s_{t}^v$ & Binary variable that shows if node $v$ receives the global model at $t$ in single-stream \ours \\
\hline
$\mathcal{S}^v_T$ & Set of $\{\sv[t]\}_{0<t\leq T}$\\
\hline
 $H$ & Synchronization bound for single-stream \ours, \ie $gap(\mathcal{S}^v_T) \leq H$, $1 \leq v \leq V$ \\
\hline
$s_{t}^v[m]$ & Binary variable that shows if node $v$ receives the global model at $t$ from stream $m$ in multi-stream \ours \\
 \hline
$\mathcal{S}^v_T[m]$ & $\{\sv[t][m]\}_{0<t\leq T}$\\
\hline
 $H_m$ & Synchronization bound  of stream $m$ in multi-stream \ours, \ie $gap(\mathcal{S}^v_T[m]) \leq H_m$, $m\in \mathcal{M}_v$ \\
\hline
$visited$ & Set of nodes that are visited for the global model update in the most recent synchronization round for single-stream \ours\\
\hline
$visited[m]$ & Set of nodes that are visited for the semi-global model update in the most recent synchronization round in stream $m$ for multi-stream \ours\\
\hline
$\xtild[t]$ & The global model received by node $v$ at $t$ in single-stream \ours\\
\hline
$\xtild[t][m]$ & The global model received by node $v$ at $t$  from stream $m$  in multi-stream \ours\\
\hline
$pre\_node$ & The node that node $v$ receives the global model from in single-stream \ours\\
\hline
$pre\_node[m]$ & The node that node $v$ receives the semi-global model from in stream $m$ for multi-stream \ours\\
\hline
$p^v_m$ & The node that node $v$ receives the global model from for the first time in the current synchronization round in stream $m$\\
\hline
$d^G_{uv}$ & The distance between $u$ and $v$, the total delay of the links in the shortest path between $u$ and $v$\\
\hline
$R^G_v$ & The greatest distance from $v$, i.e., $R^G_v =\max_v d^G_{uv}$\\
\hline
$r$ & Root or the node with the minimum $R^G_v$, i.e., $r= \arg \min_v R^G_v$ \\
\hline
$ST_r$ & The shortest path tree rooted at r\\
\hline
$P^r_v$ & Set of streams working over the path from $v$ to $r$ in $ST_r$ \\
\hline
$M$ & Number of streams in multi-stream \ours \\
\hline
$\mathcal{M}_v$ & The set of streams passing node $v$ \\
\hline
$\mathcal{V}_m$ & The set of nodes in domain of stream $m$ \\
\hline
\end{tabularx}

\newpage

\subsection{Proof of Theorem \ref{T1}}\label{appendixb}

Motivated by \cite{Stich2019LSGD}, a virtual sequence $\{\xbar\}_{t\geq 0}$ is defined as follows.
\begin{equation}
\xbar =\xz - \sumov \sum_{z=0}^{t-1} \frac{D_v}{D} \eta_z \nabla \fiv[z](\xv[z]).
\end{equation}
We do not need to calculate this sequence in the algorithm explicitly and it is only used for the sake of the analysis.
We also define
\begin{align}
&\g=  \sumov \DvtoD \nabla \fiv(\xv), &\gbar =  \sumov \DvtoD \nabla f^v(\xv),&
\end{align}
where $f(\vec{x})$, $f^v(\vec{x})$ are global loss function and local loss function in node $v$, respectively.

Let us introduce $i_t=\{i_t^1,...i_t^V\}$ to denote the data samples selected randomly during time slot $t$ in all nodes.
Then, observe that $\gbar = \EX_{i_t} \g$.
We have $\xbar[t+1] = \xbar - \eta_t \g$. 

First, we illustrate how the virtual sequence, $\{\xbar\}_{t\geq 0}$, approaches to the optimal in Lemma \ref{lem1}, and Lemma \ref{lem2}. Second, we depict in Lemma \ref{lem3} that there is a little deviation from the virtual sequence in the actual iterates, $\xv$. Finally, the convergence rate is proved.

\begin{lemma}\label{lem1}
If $\fiv(\vec{x})$ is $L$-smooth, $f^v(\vec{x})$ is $\mu$-strongly convex, $\eta_t \leq \frac{1}{4L}$, and $\EX_{i^v_t} \norm{\nabla \fiv(\xv) - \nabla f^v(\xv)}\leq \sigma^2$ for $0 \leq t \leq T-1$, $1 \leq v \leq V$ then
\begin{align}
\EX\norm{\xbar[t+1]  - \xstar} \leq (1-\mu\eta_t)\EX\norm{\xbar - \xstar} + \eta_t^2\rho\sigma^2 -\eta_t \EX(f(\xbar)-f(\xstar))+ 2L\eta_t\sumov \DvtoD \EX \norm{\xv-\xbar}.
\end{align}
\end{lemma}

\begin{proof}
We have
\begin{align}
&\norm{\xbar[t+1]-\xstar} = \norm{\xbar-\eta_t \g - \xstar} = \norm{\xbar - \xstar - \eta_t \gbar + \eta_t \gbar - \eta_t \g}\\
&= \norm{\xbar - \xstar - \eta_t \gbar} + \eta^2_t \norm{\gbar -\g } + 2\eta_t \inpr{\xbar - \xstar - \eta_t \gbar}{\gbar - \g} \label{ap_eq8}
\end{align}

Then we apply expectation to get $\EX_{i_0,...,i_t}\norm{\xbar[t+1]-\xstar}$.
Based on the law of total expectation, for every two random variables $\alpha$, $\beta$ and a function $y$, $\EX_\alpha y(\alpha) = \EX_\beta\EX_\alpha[y(\alpha)|\beta]$.
Considering $\alpha = i_0,...,i_t$ and $\beta = i_0,...,i_{t-1}$, we get that
\begin{align}
\EX_{i_0,...,i_t}\inpr{\xbar - \xstar - \eta_t \gbar}{\gbar - \g} &= \EX_{i_0,...,i_{t-1}}\EX_{i_0,...,i_t}[\inpr{\xbar - \xstar - \eta_t \gbar}{\gbar - \g}|i_0,...,i_{t-1}]\\
&= \EX_{i_0,...,i_{t-1}}\inpr{\xbar - \xstar - \eta_t \gbar}{ \gbar - \EX_{i_t}\g} \label{ap_eq11}\\
&= \EX_{i_0,...,i_{t}}\inpr{\xbar - \xstar - \eta_t \gbar}{\gbar -\EX_{i_t}\g }\\
&= \EX_{i_0,...,i_{t}}\inpr{\xbar - \xstar - \eta_t \gbar}{\gbar-\gbar}\\
&= 0. \label{ap_eq12}
\end{align}
In (\ref{ap_eq11}), we used the fact that once we know $i_0,...,i_{t-1}$, the value of $\xv$, $1\leq v \leq V$, and therefore $\xbar$ and $\gbar$ are not random any more.
From now on, we drop the subscript $i_0,...,i_{t}$ for the ease of notation. Thus,
\begin{align}
\EX \norm{\xbar[t+1]-\xstar} \leq & \EX \norm{\xbar - \xstar - \eta_t \gbar }+\eta^2_t \EX \norm{\gbar - \g}. \label{ap_eq14}
\end{align}
We obtain
\begin{align}
\norm{\xbar - \xstar - \eta_t \gbar } &= \norm{\xbar - \xstar} + \eta_t^2 \norm{\gbar} -2 \eta_t \inpr{\xbar - \xstar}{\gbar}\\
&= \norm{\xbar - \xstar} + \eta_t^2 \norm{\gbar} - 2\eta_t \sumov \DvtoD \inpr{\xbar - \xv + \xv - \xstar}{\nabla f^v(\xv)}\\
&=  \norm{\xbar - \xstar} + \eta^2  \norm{\sumov \DvtoD\nabla f^v(\xv)} - 2\eta_t \sumov \DvtoD \inpr{\xv - \xstar}{\nabla f^v(\xv)}\label{eq18}\\
&& \mathllap{-2\eta \sumov \DvtoD \inpr{\xbar - \xv}{ \nabla f^v(\xv)}.}\nonumber
\end{align}
We also get
\begin{align}
  \norm{\sumov \DvtoD \nabla f^v(\xv)}&=\norm{\sumov \DvtoD (\nabla f^v(\xv)-\nabla f^v(\xbar)+\nabla f^v(\xbar)-\nabla f^v(\xstar))}\\
  &\leq 2\norm{\sumov \DvtoD (\nabla f^v(\xv)-\nabla f^v(\xbar))}+2\norm{\sumov \DvtoD (\nabla f^v(\xbar)-\nabla f^v(\xstar))}\label{eq20}\\
  &\leq 2\norm{\sumov \DvtoD (\nabla f^v(\xv)-\nabla f^v(\xbar))}+2\norm{\nabla f(\xbar)-\nabla f(\xstar)}\label{eq21}\\
  &\leq 2\sumov \DvtoD \norm{(\nabla f^v(\xv)-\nabla f^v(\xbar))}+2\norm{(\nabla f(\xbar)-\nabla f(\xstar))}\label{eq22}\\
  &\leq 2L^2\sumov \DvtoD \norm{\xv-\xbar}+4L(f(\xbar)-f(\xstar)),\label{eq23}
\end{align}
where (\ref{eq20}) is based on the following inequality.
\begin{align}
    \norm{\sum_{i=1}^n a_i} \leq n \sum_{i=1}^n \norm{a_i}.\label{ineq_sum}
\end{align}
In (\ref{eq21}) we have used the fact that $\sumov \DvtoD f^v(\vec{x}) = f(\vec{x})$.
(\ref{eq22}), and (\ref{eq23}) are due to the convexity of $\norm{\cdot}$ and L-smoothness, respectively.
Note that by $L$-smoothness we have
\begin{align}
\norm{\nabla f(\vec{x}) - \nabla f(\xstar)} \leq 2L(f(\vec{x}) - f^*).\label{eq24}
\end{align}
$\mu$-strong convexity provides us with
\begin{align}
    -\inpr{\xv - \xstar}{\nabla f^v(\xv)} \leq -(f^v(\xv)-f^v(\xstar)) - \frac{\mu}{2} \norm{\xv-\xstar}.\label{eq26}
\end{align}
Using $L$-smoothness to bound the last term in (\ref{eq18}), we have
\begin{align}
    -\inpr{\xbar - \xv}{\nabla f^v(\xv)} &\leq f^v(\xv) - f^v(\xbar) + \frac{L}{2} \norm{\xv - \xbar}\label{eq27}
\end{align}
We obtain the following result by applying (\ref{eq23}), (\ref{eq26}), and (\ref{eq27}) to (\ref{eq18}):
\begin{align}
\norm{\xbar - \eta_t \gstar - \xstar} \leq \norm{\xbar - \xstar} + 2L^2 \eta_t^2\sumov \DvtoD \norm{\xv-\xbar}+4L\eta_t^2(f(\xbar)-f(\xstar)) \\
&& \mathllap{+ 2\eta_t \sumov \DvtoD \big(-(f^v(\xv)-f^v(\xstar)) - \frac{\mu}{2} \norm{\xv-\xstar} + f^v(\xv) - f^v(\xbar) + \frac{L}{2} \norm{\xv - \xbar}\big)}.\nonumber
\end{align}
This can be rewritten as 
\begin{align}
\norm{\xbar - \eta_t \gstar - \xstar} \leq \norm{\xbar - \xstar} +2\eta(2L\eta_t-1)(f(\xbar)-f(\xstar))+ L \eta_t(1+2L \eta_t)\sumov \DvtoD \norm{\xv-\xbar}.\label{eq28}\\
&& \mathllap{-\mu\eta_t \sumov \DvtoD \norm{\xv-\xstar}.}\nonumber
\end{align}
Using concavity of $\beta\norm{\cdot}$ for $ \beta \leq 0$, we get
\begin{align}
    -\mu \sumov \DvtoD \norm{\xv-\xstar} \leq -\mu\norm{\xbar-\xstar},
\end{align}
so, we get
\begin{align}
\norm{\xbar - \eta_t \gstar - \xstar} \leq (1-\mu\eta_t)\norm{\xbar - \xstar} +2\eta(2L\eta_t-1)(f(\xbar)-f(\xstar))+ L \eta_t(1+2L \eta_t)\sumov \DvtoD \norm{\xv-\xbar}.\label{eq30}
\end{align}
We have $(2\eta_tL-1) \leq -\frac{1}{2}$, $(2\eta_tL+1) \leq 2$  as we assumed $\eta_t \leq \frac{1}{4L}$. So, we obtain
\begin{align}
\norm{\xbar - \eta_t \gstar - \xstar} \leq (1-\mu\eta_t)\norm{\xbar - \xstar} -\eta(f(\xbar)-f(\xstar))+ 2L\eta_t\sumov \DvtoD \norm{\xv-\xbar}.\label{eq31}
\end{align}

By definition, we have that
\begin{align}
\EX \norm{\g - \gbar} &=\EX \norm{\sumov \DvtoD (\nabla \fiv(\xv) - \nabla f^v(\xv))}\\
&= \sumov (\DvtoD)^2\EX \norm{(\nabla \fiv(\xv) - \nabla f^v(\xv)} \label{ap_eq34}\\
&=\sigma^2 \sumov(\DvtoD)^2 \\
&= \rho\sigma^2  \label{bound_var},
\end{align}
where (\ref{ap_eq34}) is based on the fact that variance of the sum of independent random variables equals sum of their variances.
Taking expectation of (\ref{eq31}) and applying it with (\ref{bound_var}) into (\ref{ap_eq14}) provides
\begin{align}
\EX\norm{\xbar[t+1]  - \xstar} \leq (1-\mu\eta_t)\EX\norm{\xbar - \xstar} + \eta_t^2\rho\sigma^2 -\eta_t\EX(f(\xbar)-f(\xstar))+ 2L\eta_t\sumov \DvtoD \EX \norm{\xv-\xbar}.\label{eq36}
\end{align}
\end{proof}

\begin{lemma}\label{lem2}
If $f^v(\vec{x})$ is $L$-smooth , $\eta_t \leq \frac{1}{4L}$, and $\EX_{i^v_t} \norm{\nabla \fiv(\xv) - \nabla f^v(\xv)}\leq \sigma^2$ for $0 \leq t \leq T-1$, $1 \leq v \leq V$ then
\begin{align}
\EX f(\xbar[{t+1}]) \leq \EX f(\xbar) + \frac{\eta_t^2L\rho\sigma^2}{2} -\frac{\eta_t}{4} \EX\norm{\nabla f(\xbar)}+ L^2\eta_t\sumov \DvtoD \EX \norm{\xv-\xbar}.
\end{align}
\end{lemma}

\begin{proof}
Based on the defenition of $\xbar$ and L-smoothness of $f^v(\vec{x})$ we have
\begin{align}
    f(\xbar[t+1]) &= f(\xbar - \eta_t \g) \\
    &\leq f(\xbar) + \eta_t \inpr{\nabla f(\xbar)}{-\g} + \frac{\eta_t^2L}{2} \norm{\g}. \label{eq39}
\end{align}
Let's take expectation of the second term on the right-hand side of (\ref{eq39})
\begin{align}
    \eta_t\EX \inpr{\nabla f(\xbar)}{-\g} &= \eta_t\EX \inpr{\nabla f(\xbar)}{\nabla f(\xbar)-\nabla f(\xbar)-\gbar }\\
    &= -\eta_t\EX \norm{\nabla f(\xbar)} + \eta_t \EX \inpr{\nabla f(\xbar)}{\nabla f(\xbar) - \gbar}\\
    & \leq -\eta_t\EX \norm{\nabla f(\xbar)} +\frac{\eta_t}{2}\EX \norm{\nabla f(\xbar)} + \frac{\eta_t}{2}  \EX\norm{\nabla f(\xbar)-\gbar} \label{eq43}\\
    & \leq -\frac{\eta_t}{2}\EX \norm{\nabla f(\xbar)} + \frac{\eta_t}{2}  \EX\norm {\sumov \DvtoD (\nabla f^v(\xbar)-\nabla f^v(\xv))} \\
    & \leq -\frac{\eta_t}{2}\EX \norm{\nabla f(\xbar)} + \frac{\eta_t}{2} \sumov \DvtoD \EX\norm { \nabla f^v(\xbar)-\nabla f^v(\xv)} \label{eq45} \\
    & \leq -\frac{\eta_t}{2}\EX \norm{\nabla f(\xbar)} + \frac{\eta_tL^2}{2}  \sumov \DvtoD\EX\norm { \xbar -\xv}.\label{eq46} 
\end{align}
(\ref{eq43}) is based on the fact that for any $\lambda > 0$,
\begin{align}
2\inpr{a}{b} \leq \lambda \norm{a} + \frac{1}{\lambda} \norm{b}.\label{lambda_ineq}
\end{align} 
(\ref{eq45}) is based on the convexity of $\norm{.}$ and (\ref{eq46}) is due to $L$-smoothness.
Let's take expectation of the last term on the right-hand side of (\ref{eq39}) as
\begin{align}
    \frac{\eta_t^2L}{2} \EX \norm{\g} &\leq   \frac{\eta_t^2L}{2} \EX \norm{\g - \gbar} +   \frac{\eta_t^2L}{2}  \EX\norm{\gbar}\\
    &\leq  \frac{\eta_t^2L}{2} \rho \sigma^2 +  \eta_t^2 L^3\sumov \DvtoD \EX\norm{\xv-\xbar}+ \eta_t^2L \EX\norm{\nabla f(\xbar)} , \label{eq49}
\end{align}
where (\ref{eq49}) is based on (\ref{eq23}), and (\ref{bound_var}).
 Putting everything together, we obtain
\begin{align}
    \EX f(\xbar[{t+1}]) \leq \EX f(\xbar) +\frac{\eta_t^2L\rho\sigma^2}{2} +\eta_t (\eta_t L - \frac{1}{2})\EX\norm{\nabla f(\xbar)}+ L^2\eta_t(\eta_t L + \frac{1}{2})\sumov \DvtoD \EX \norm{\xv-\xbar}.
\end{align}
Considering $\eta_t \leq \frac{1}{4L}$ we obtain
\begin{align}
\EX f(\xbar[{t+1}]) \leq \EX f(\xbar) + \frac{\eta_t^2L\rho\sigma^2}{2} -\frac{\eta_t}{4} \EX\norm{\nabla f(\xbar)}+ L^2\eta_t\sumov \DvtoD \EX \norm{\xv-\xbar}.
\end{align}
\end{proof}
Observe that Lemmas \ref{lem1}, \ref{lem2} hold regardless of how to synchronize the nodes, i.e., hold for both single and multi-stream \ours.


Before going to the next lemma, we define $\tau$-slow sequences \cite{Stich&Karimireddy}, $\{a_t\}_{t\geq 0}$ of positive values is $\tau$-slow decreasing for parameter $\tau\geq 1$ if
\begin{align}
    a_t \geq a_{t+1}, \qquad a_t \leq a_{t+1} (1 + \frac{1}{2\tau}), \qquad t\geq 0.
\end{align}
The sequence $\{a_t\}_{t\geq 0}$ is $\tau$-slow increasing if $\{a_t^{-1}\}_{t\geq 0}$
is $\tau$-slow decreasing.

\begin{lemma}[Bounding deviation]\label{lem3}
If $\max\{\lv-t\}\leq E$, $\EX_{i^v_t} \norm{\nabla \fiv(\xv) - \nabla f^v(\xv)}\leq \sigma^2$ ,$ \norm{\nabla f(\xv) - \nabla f^v(\xv)}\leq \zeta^2$, $\eta_t=\eta \leq \frac{1}{30L(H'+E)}$, and $\omega_{t}$ is $(H'+E)$-slow increasing for $0 \leq t \leq T-1$, $1 \leq v \leq V$, and for (i) single-stream \ours $gap(\mathcal{S}^v_T) \leq H$, $1 \leq v \leq V$,
(ii) multi-stream \ours, $gap(\mathcal{S}^v_T[m]) \leq H_m$ for $v$ that $m \in \mathcal{M}_v$,
\begin{align}
    \sum_{t=0}^{T-1} \omega_t \sumov \DvtoD \EX \norm{\xbar - \xv} \leq \frac{1}{8L^2}\sum_{t=0}^{T-1} \omega_t  \EX\norm{\nabla f(\xbar)}+90\eta^2 (H'+E) \big( \zeta^2 (H'+E) + \sigma^2\big)  \sum_{t=0}^{T-1}  \omega_t.\label{ap_eq_lem2}
\end{align}
where in single-stream $H'=H$ and in multi-stream \ours $H'=\max_{u}\sum_{m\in p_u^v}H_m$.
\end{lemma}

\begin{proof}
First we focus on the single stream case, $\tav[t]$ denotes the last time slot up to $t$, when node $v$'s model was updated with the synchronization stream, \ie $\tav[t] = \max \{t'\mid  t'\leq t, s_{t'}^v=1\}$.
Lets use (\ref{ineq_sum}) to decompose the the deviation term as depicted in the following:
\begin{align}
    \norm{\xbar - \xv} \leq 2(\norm{\xv- \xtild[\tau_t^v]} + \norm{\xbar-\xtild[\tau_t^v]}). \label{ap_eq42}
\end{align}    
For the first term we can obtain
\begin{align}
    \EX \norm{\xv-\xtild[\tau_t^v]} &= \EX \norm{\xv-\xv[\tau_t^v]}\\
    &= \EX \norm{\sum_{t'=\tau_t^v}^{t-1} \sum_{z \in \uv[t']}  \eta_z \nabla \fiv[z](\xv[z])} \\
    &= \EX \norm{ \sum_{z \in \cup_{t'=\tau_t^v}^{t-1} \uv[t']} \eta_z \nabla \fiv[z](\xv[z])} .\label{ap_eq45}
\end{align}
And we have 
\begin{align}
    \nabla \fiv[z](\xv[z]) &= \big(\nabla \fiv[z]
    (\xv[z])- \nabla f^v(\xv[z])\big) +\nabla f^v(\xv[z])\\
    &=\big(\nabla \fiv[z]
    (\xv[z])- \nabla f^v(\xv[z])\big) + \big(\nabla f^v(\xv[z]) -  \nabla f^v(\xbar[z])\big)  +  \nabla f^v(\xbar[z])\\
    &=\big(\nabla \fiv[z]
    (\xv[z])- \nabla f^v(\xv[z])\big) + \big(\nabla f^v(\xv[z]) -  \nabla f^v(\xbar[z])\big)  +  \big(\nabla f^v(\xbar[z]) - \nabla f(\xbar[z]) \big) + \nabla f(\xbar[z]),
\end{align}
so we can rewrite the right-hand side of (\ref{ap_eq45}) as
\begin{align}
    \EX \norm{ \sum_{\mathcal{U}^v_t} \eta_z \nabla \fiv[z](\xv[z])} &= 4 \EX \Big(\norm{ \sum_{z \in \mathcal{U}^v_t} \eta_z \nabla f(\xbar[z])}+ \norm{ \sum_{z \in \mathcal{U}^v_t} \eta_z \big(\nabla f^v(\xv[z]) -  \nabla f^v(\xbar[z])\big)} \label{ap_eq47}\\
    && \mathllap{+ \norm{ \sum_{z \in \mathcal{U}^v_t} \eta_z \big(\nabla f^v(\xbar[z]) - \nabla f(\xbar[z]) \big)}+  \norm{ \sum_{z \in \mathcal{U}^v_t} \eta_z \big(\nabla \fiv[z] (\xv[z])- \nabla f^v(\xv[z])\big)} \Big)},\nonumber 
\end{align}
where $\mathcal{U}^v_t =  \cup_{t'=\tau_t^v}^{t-1} \uv[t']$.
Now we sum over $v$, multiply by $\omega_t$, sum over $t$, and bound every term in (\ref{ap_eq47}) as follows. Considering that we have $|\mathcal{U}^v_t| \leq H+E$, we obtain
\begin{align}
    \sum_{t=0}^{T-1} \omega_t \sumov \DvtoD \norm{ \sum_{z \in \mathcal{U}^v_t} \eta_z \nabla f(\xbar[z])}& \leq \sum_{t=0}^{T-1} \omega_t \sumov \DvtoD |\mathcal{U}^v_t|  \sum_{z \in\mathcal{U}^v_t} \norm{\eta_z \nabla f(\xbar[z])} \\
    &\leq (H+E)\sum_{t=0}^{T-1} \omega_t \sumov \DvtoD  \sum_{z \in \mathcal{U}^v_t} \eta_{z}^2\norm{\nabla f(\xbar[z])} \\
    &\leq 2(H+E) \sumov \DvtoD\ \sum_{t=0}^{T-1}  \sum_{z \in \mathcal{U}^v_t} \omega_z \eta_{z}^2 \norm{\nabla f(\xbar[z])}\label{ap_eq50} \\
    &\leq 2(H+E)^2 \sumov \DvtoD\ \sum_{t=0}^{T-1} \omega_t \eta_{t}^2 \norm{\nabla f(\xbar)}\\
    &\leq 2(H+E)^2 \sum_{t=0}^{T-1} \omega_t \eta_{t}^2 \norm{\nabla f(\xbar)},
\end{align}
where (\ref{ap_eq50}) is due to the fact that $\omega_t \leq \omega_{t-z}(1+\frac{1}{2(H+E)})^z \leq\omega_{t-z}(1+\frac{1}{2(H+E)})^{H+E} \leq \omega_{t-z} \exp{(\frac{1}{2})} \leq 2\omega_{t-z}$ for $ 0\leq z \leq H+E$, i.e., $\omega_{t}$ is $(H+E)$-slow increasing.
The second term in the right-hand side of (\ref{ap_eq47}) is bounded as
\begin{align}
    \sum_{t=0}^{T-1} \omega_t \sumov \DvtoD \norm{ \sum_{z \in \mathcal{U}^v_t} \eta_z & \big(\nabla f^v(\xv[z]) -  \nabla f^v(\xbar[z])\big)} \\
    &\leq \sum_{t=0}^{T-1} \omega_t \sumov \DvtoD |\mathcal{U}^v_t|  \sum_{z \in\mathcal{U}^v_t} \norm{\eta_z \big(\nabla f^v(\xv[z]) -  \nabla f^v(\xbar[z])\big)} \\
    &\leq (H+E)\sum_{t=0}^{T-1} \omega_t \sumov \DvtoD  \sum_{z \in \mathcal{U}^v_t} \eta_{z}^2\norm{\big(\nabla f^v(\xv[z]) -  \nabla f^v(\xbar[z])\big)} \\
    &\leq 2(H+E) \sumov \DvtoD\ \sum_{t=0}^{T-1}  \sum_{z \in \mathcal{U}^v_t} \omega_z \eta_{z}^2 \norm{\big(\nabla f^v(\xv[z]) -  \nabla f^v(\xbar[z])\big)} \\
    &\leq 2(H+E)^2 \sumov \DvtoD\ \sum_{t=0}^{T-1} \omega_t \eta_{t}^2 \norm{\big(\nabla f^v(\xv) -  \nabla f^v(\xbar)\big)}\\
    &\leq 2(H+E)^2 L^2 \sum_{t=0}^{T-1} \omega_t \eta_{t}^2  \sumov \DvtoD\  \norm{\xbar-\xv}.
\end{align}
For the third term in in the right-hand side of (\ref{ap_eq47}), we obtain

\begin{align}
    \sum_{t=0}^{T-1} \omega_t \sumov \DvtoD \norm{ \sum_{z \in \mathcal{U}^v_t}& \eta_z \big(\nabla f^v(\xbar[z]) - \nabla f(\xbar[z]\big)}\\
    & \leq \sum_{t=0}^{T-1} \omega_t \sumov \DvtoD |\mathcal{U}^v_t| \sum_{z \in\mathcal{U}^v_t} \eta_{z}^2 \norm{\big(\nabla f^v(\xbar[z]) - \nabla f(\xbar[z])\big)}\\
    & \leq 2(H+E)\sumov \DvtoD  \sum_{t=0}^{T-1}  \sum_{z \in\mathcal{U}^v_t} \omega_z\eta_{z}^2 \zeta^2\\
    & \leq 2(H+E)^2\sumov \DvtoD  \sum_{t=0}^{T-1}  \omega_t\eta_{t}^2 \zeta^2\\
    & \leq 2(H+E)^2\zeta^2  \sum_{t=0}^{T-1}  \omega_t\eta_{t}^2 .
\end{align}

Using the same approach for the last term, we can obtain
\begin{align}
    \sum_{t=0}^{T-1} \omega_t \sumov \DvtoD  \EX \norm{ \sum_{z \in \mathcal{U}^v_t}& \eta_z \big(\nabla \fiv[z]^v(\xv[z]) - \nabla f^v(\xv[z])\big)} \leq 2(H+E)\sigma^2  \sum_{t=0}^{T-1}  \omega_t\eta_{t}^2 ,
\end{align}
where instead of (\ref{ineq_sum}) we have used the fact that for independent zero-mean random variables, we get a tighter bound as follows.
\begin{align}
    \EX \norm{\sum_{i=1}^n a_i} \leq \sum_{i=1}^n \EX \norm{a_i}.\label{zero-mean_ineq_sum}
\end{align}

Adding up the last four inequalities and applying them in (\ref{ap_eq47}) and the final result in (\ref{ap_eq45}), we get
\begin{align}
    \sum_{t=0}^{T-1} \omega_t \sumov \DvtoD \EX \norm{\xv-\xtild[\tau_t^v]} \leq 8(H+E)^2 L^2 \sum_{t=0}^{T-1} \omega_t \eta_{t}^2  \sumov \DvtoD\  \EX\norm{\xbar-\xv} \label{ap_eq58}\\
    && \mathllap{+8(H+E)^2 \sum_{t=0}^{T-1} \omega_t \eta_{t}^2 \EX\norm{\nabla f(\xbar)}+8(H+E)^2\zeta^2  \sum_{t=0}^{T-1}  \omega_t\eta_{t}^2  + 8(H+E)\sigma^2  \sum_{t=0}^{T-1}  \omega_t\eta_{t}^2 }. \nonumber
\end{align}

Now we try to bound the second term in (\ref{ap_eq42}).
\begin{align}
    \norm{\xbar-\xtild[\tau_t^v]} \leq 2  \big(\norm{\xbar-\xtild} + \norm{\xtild-\xtild[\tau_t^v]}\big), \label{ap_eq59}
\end{align}
where we have 
\begin{align}
    \EX\norm{\xbar-\xtild} &= \EX\norm{\sumov \sum_{t'=\tau_t^v}^{t-1} \sum_{z \in \uv[t']} \DvtoD \eta_z \nabla f_{i^v_z}(\vec{x}^v_z)}\\
    &\leq \sumov  \DvtoD \EX\norm{ \sum_{t'=\tau_t^v}^{t-1} \sum_{z \in \uv[t']} \eta_z \nabla f_{i^v_z}(\vec{x}^v_z)} \label{ap_eq63}\\
    & = \sumov  \DvtoD \EX \norm{ \sum_{\mathcal{U}^v_t} \eta_z \nabla f_{i^v_z}(\vec{x}^v_z)}, \label{ap_eq62}
\end{align}
where (\ref{ap_eq63}) is due to the convexity of $\norm{}$. 

We can write
\begin{align}
    \EX\norm{\xtild-\xtild[\tau_t^v]} &= \EX \norm{\sum_{h \in \mathcal{H}^v_t} \sum_{t'=\tau^h_{\tau_t^v}}^{\tau_t^h-1} \sum_{z \in \uv[t']} \DvtoD[h] \eta_z \nabla f_{i^h_z}(\vec{x}^h_z)}\\
    & \leq \sum_{h \in \mathcal{H}^v_t} \DvtoD[h] \EX \norm{ \sum_{t'=\tau^h_{\tau_t^v}}^{\tau_t^h-1} \sum_{z \in \uv[t']}  \eta_z \nabla f_{i^h_z}(\vec{x}^h_z)}\\
    & \leq \sumov  \DvtoD \EX \norm{ \sum_{t'=\tau^h_{\tau_t^v}}^{\tau_t^h-1} \sum_{z \in \uv[t']}  \eta_z \nabla f_{i^h_z}(\vec{x}^h_z)}\label{ap_eq65}
\end{align}
where $\mathcal{H}^v_t = \{h \mid  \tau_t^v \leq \tau_t^h \leq t \}$ is the set of nodes that are visited after node $v$ in the current synchronization round.

(\ref{ap_eq65}), and (\ref{ap_eq62}) can be extended exactly as in (\ref{ap_eq47}) and according to (\ref{ap_eq59}) we will have
\begin{align}
    \sum_{t=0}^{T-1} \omega_t \sumov \DvtoD \EX \norm{\xbar-\xtild[\tau_t^v]} \leq 4 \bigg(8(H+E)^2 L^2 \sum_{t=0}^{T-1} \omega_t \eta_{t}^2  \sumov \DvtoD\  \EX\norm{\xbar-\xv}\label{ap_eq80}\\
    && \mathllap{+8(H+E)^2 \sum_{t=0}^{T-1} \omega_t \eta_{t}^2 \EX\norm{\nabla f(\xbar)} +8(H+E)^2\zeta^2  \sum_{t=0}^{T-1}  \omega_t\eta_{t}^2  + 8(H+E)\sigma^2  \sum_{t=0}^{T-1}  \omega_t\eta_{t}^2 }\bigg). \nonumber
\end{align}

Using (\ref{ap_eq58}), and (\ref{ap_eq80}) in (\ref{ap_eq42}) we obtain
\begin{align}
    \sum_{t=0}^{T-1} \omega_t \sumov \DvtoD \EX \norm{\xbar - \xv} \leq 10 \bigg(8(H+E)^2 L^2 \sum_{t=0}^{T-1} \omega_t \eta_{t}^2  \sumov \DvtoD\  \EX\norm{\xbar-\xv}\label{ap_eq81}\\
    && \mathllap{+8(H+E)^2 \sum_{t=0}^{T-1} \omega_t \eta_{t}^2 \EX\norm{\nabla f(\xbar)} +8(H+E)^2\zeta^2  \sum_{t=0}^{T-1}  \omega_t\eta_{t}^2  + 8(H+E)\sigma^2  \sum_{t=0}^{T-1}  \omega_t\eta_{t}^2 }\bigg). \nonumber
\end{align}
By rearranging (\ref{ap_eq81}) and assuming $\eta_t = \eta$, we get
\begin{align}
    \sum_{t=0}^{T-1} \omega_t \sumov \DvtoD \EX \norm{\xbar - \xv} \leq \frac{1}{(1-80\eta^2(H+E)^2L^2)}\bigg(80(H+E)^2 \eta^2 \sum_{t=0}^{T-1} \omega_t \EX\norm{\nabla f(\xbar)} \label{ap_eq82}\\
    && \mathllap{+80\eta^2 (H+E)^2\zeta^2  \sum_{t=0}^{T-1}  \omega_t + 80\eta^2 (H+E)\sigma^2  \sum_{t=0}^{T-1}  \omega_t }\bigg). \nonumber
\end{align}
Let $\eta \leq \frac{1}{30L(H+E)}$ to get
\begin{align}
    \sum_{t=0}^{T-1} \omega_t \sumov \DvtoD \EX \norm{\xbar - \xv} \leq \frac{1}{8L^2}\sum_{t=0}^{T-1} \omega_t  \EX\norm{\nabla f(\xbar)}+90\eta^2 (H+E) \big( \zeta^2 (H+E) + \sigma^2\big)  \sum_{t=0}^{T-1}  \omega_t.
\end{align}

This completes the proof for the single-stream \ours. A similar argument can be made for multi-stream \ours with just one subtle change.
Note that $2(H+E)$ indicates how long it takes for an SGD update performed in one node to be available in all the other nodes. 
In multi-stream \ours, this is determined by the depth of the tree, i.e., the longest delay path from the root node to its farthest leaf (in terms of the delay distance), which is expressed as $2(E+\max_{u}\sum_{m\in P_u^r}H_m)$. So, replacing $(H+E)$ with $(E+\max_{u}\sum_{m\in P_u^r}H_m)$ gives the result for multi-stream \ours.
\end{proof}

Combining Lemma \ref{lem3} with Lemma \ref{lem1} for the convex case and Lemma \ref{lem2} for the non-convex case, we can obtain a recursive description of suboptimality. We follow closely the technique described in \cite{Koloskova2020Unified} for estimating the convergence rates.

\textbf{Convex}

Based on lemma \ref{lem1} we have 
\begin{align}
\EX\norm{\xbar[t+1]  - \xstar} \leq (1-\mu\eta_t)\EX\norm{\xbar - \xstar} + \eta_t^2\rho\sigma^2 - \eta \EX(f(\xbar)-f(\xstar))+ 2L\eta_t\sumov \DvtoD \EX \norm{\xv-\xbar}.
\end{align}
By assuming $\eta_t = \eta$ and multiplication of $\frac{\omega_t}{\eta}$ in both sides and summing up we get
\begin{align}
\sum_{t=0}^{T-1}\frac{\omega_t}{\eta}\EX\norm{\xbar  - \xstar} \leq \sum_{t=0}^{T-1}\frac{\omega_t(1-\mu\eta)}{\eta}\EX\norm{\xbar - \xstar} + \eta\rho\sigma^2 \sum_{t=0}^{T-1}\omega_t - \sum_{t=0}^{T-1}\omega_t \EX(f(\xbar)-f(\xstar))+ \\
&& \mathllap{2L\sum_{t=0}^{T-1}\omega_t\sumov \DvtoD \EX \norm{\xv-\xbar}.}\nonumber
\end{align}
By replacing result of lemma \ref{lem3} we get
\begin{align}
\sum_{t=0}^{T-1}\frac{\omega_t}{\eta}\EX\norm{\xbar  - \xstar} \leq \sum_{t=0}^{T-1}\frac{\omega_t(1-\mu\eta)}{\eta}\EX\norm{\xbar - \xstar} + \eta\rho\sigma^2 \sum_{t=0}^{T-1}\omega_t - \sum_{t=0}^{T-1}\omega_t\EX(f(\xbar)-f(\xstar))+ \\
&& \mathllap{2L\big(\frac{1}{8L^2}\sum_{t=0}^{T-1} \omega_t  \EX\norm{\nabla f(\xbar)}+90\eta^2 (H+E) \big( \zeta^2 (H+E) + \sigma^2\big)  \sum_{t=0}^{T-1}  \omega_t\big).}\nonumber
\end{align}
By using (\ref{eq24}), dividing both sides by $W_T = \sum_{t=0}^{T-1} \omega_t$, and rearranging, we have
\begin{align}
\frac{3}{4W_T}\sum_{t=0}^{T-1}\omega_t \EX(f(\xbar)-f(\xstar)) \leq \frac{1}{W_T\eta}\sum_{t=0}^{T-1}\big(\omega_t(1-\mu\eta)\EX\norm{\xbar - \xstar}-\omega_t\EX\norm{\xbar  - \xstar} \big)+ \eta\rho\sigma^2 + \\
&& \mathllap{180L\eta^2 (H+E) \big( \zeta^2 (H+E) + \sigma^2\big).}\nonumber
\end{align}
Based on the convexity of $f$ we have
\begin{align}
    \frac{3}{4}(\EX f(\xhat) - f^*) &\leq \frac{3}{4W_T}\sum_{t=0}^{T-1}\omega_t(\EX f(\xbar)-f^*)\\
    &\leq \frac{3}{4W_T\eta}\sum_{t=0}^{T-1}\big(\omega_t(1-\mu\eta)\EX\norm{\xbar - \xstar}-\omega_t\EX\norm{\xbar  - \xstar} \big)+ \eta\rho\sigma^2 + \label{eq96}\\
&& \mathllap{180L\eta^2 (H+E) \big( \zeta^2 (H+E) + \sigma^2\big).}\nonumber
\end{align}
Now we state two lemmas that helps us bound the right-side of (\ref{eq96}).
\begin{lemma}[Similar to Lemma 15 in \cite{Koloskova2020Unified}]\label{lem4}
    For every non-negative sequence $\{r_t\}_{t\geq0}$ and any parameters $d \geq a \geq 0, b\geq 0, c\geq 0, T\geq 0$, there exist a constant $\eta \leq \frac{1}{d}$, such that for weights $\omega_t = (1-a\eta)^{-(t+1)}$ it holds
    \begin{align}
        \frac{1}{W_T\eta}\sum_{t=0}^{T-1}\big(\omega_t(1-a\eta)r_t-\omega_t r_{t+1} \big)+ b\eta + c\eta^2 \leq \Tilde{O}(r_0d \exp(\frac{-aT}{d})+\frac{b}{aT} + \frac{c}{a^2T^2}),
    \end{align}
    where $\Tilde{O}$ hides polylogarithmic factors.
\end{lemma}
\begin{proof}
    Considering that $\omega_t(1-a\eta) =\omega_{t-1}$ we obtain a telescopic sum
        \begin{align}
        \frac{1}{W_T\eta}\sum_{t=0}^{T-1}\big(\omega_t(1-a\eta)r_t-\omega_t r_{t+1} \big)+ b\eta + c\eta^2 &\leq \frac{1}{W_T\eta} (1-a\eta) \omega_0 r_0 +b\eta + c\eta^2\\
        &\leq \frac{1}{W_T\eta} r_0 +b\eta + c\eta^2 \label{eq99}\\
        &\leq \frac{r_0}{\eta}\exp(-a\eta T) +b\eta + c\eta^2 \label{eq100},
    \end{align}
    where (\ref{eq100}) is based on $W_T \geq \omega_T \geq (1-a\eta)^{-T}\geq \exp(a\eta T)$.
    It is now followed by a $\eta$-tuning, the same way as in \cite{stich2019unified}, which shows we need to choose $\eta = \min\{\frac{1}{d},\frac{ln(\max\{2,a^2r_0T^2/b\})}{aT}\}$.
\end{proof}

\begin{lemma}[Similar to Lemma 16 in \cite{Koloskova2020Unified}]\label{lem5}
    For every non-negative sequence $\{r_t\}_{t\geq0}$ and any parameters $d \geq 0, b\geq 0, c\geq 0, T\geq 0$, there exist a constant $\eta \leq \frac{1}{d}$, it holds
    \begin{align}
        \frac{1}{(T+1)\eta}\sum_{t=0}^{T-1}\big(r_t- r_{t+1} \big)+ b\eta + c\eta^2 \leq \frac{2\sqrt{br_0}}{\sqrt{T+1}} + 2 (\frac{r_0\sqrt{c}}{T+1})^{\frac{2}{3}} + \frac{dr_0}{T+1}. 
    \end{align}
\end{lemma}
\begin{proof}
    By canceling the same terms in the telescopic sum we get
        \begin{align}
        \frac{1}{(T+1)\eta}\sum_{t=0}^{T-1}\big(r_t- r_{t+1} \big)+ b\eta + c\eta^2 \leq \frac{r_0}{(T+1)\eta}+ b\eta + c\eta^2.
    \end{align}

It is now followed by a $\eta$-tuning, the same way as in \cite{Koloskova2020Unified}, which shows we need to choose $\eta = \min\{\frac{1}{d},\sqrt{\frac{r_0}{b(T+1)}},(\frac{r_0}{c(T+1)})^{\frac{1}{3}}\}$.
\end{proof}

\textbf{(strongly convex case)}

Combining (\ref{eq96}) with $\mu>0$, and Lemma \ref{lem4} and considering that $\eta_t=\eta \leq \frac{1}{30L(H+E)}$, provides
\begin{align}
    \EX f(\xhat) - f^* &\leq \Tilde{O} \bigg(\norm{\xz-\xstar}L(H+E) \exp(\frac{-\mu T}{L(H+E)}) + \frac{\rho \sigma^2}{\mu T}+\frac{L(H+E)(\sigma^2 + \zeta^2(H+E))}{\mu^2 T^2}\bigg).
\end{align}

\textbf{(convex case)}

Combining (\ref{eq96}) with $\mu=0$, and Lemma \ref{lem5} and considering that $\eta_t=\eta \leq \frac{1}{30L(H+E)}$, provides
\begin{align}
    \EX f(\xhat) - f^* &\leq O \bigg(\frac{\norm{\xz-\xstar}L(H+E)}{T} + \frac{\sigma||{\xz-\xstar}||\sqrt{\rho}}{\sqrt{T}} + (\frac{\norm{\xz-\xstar}\sqrt{L(H+E)(\sigma^2 + \zeta^2(H+E))}}{T})^{\frac{2}{3}} \bigg).
\end{align}

\textbf{Non-convex}

Based on lemma \ref{lem2} we have 
\begin{align}
\EX f(\xbar[{t+1}]) \leq \EX f(\xbar) + \frac{\eta_t^2L\rho\sigma^2}{2} -\frac{\eta_t}{4} \EX\norm{\nabla f(\xbar)}+ L^2\eta_t\sumov \DvtoD \EX \norm{\xv-\xbar}.
\end{align}
By assuming $\eta_t = \eta$ and multiplication of $\frac{\omega_t}{\eta}$ in both sides and summing up we get
\begin{align}
\sum_{t=0}^{T-1}\frac{\omega_t}{\eta}\EX f(\xbar[{t+1}]) \leq \sum_{t=0}^{T-1}\frac{\omega_t}{\eta}\EX f(\xbar) + \frac{L\eta\rho\sigma^2}{2} \sum_{t=0}^{T-1}\omega_t - \frac{1}{4}\sum_{t=0}^{T-1}\omega_t \EX\norm{\nabla f(\xbar)}+ \\
&& \mathllap{L^2\sum_{t=0}^{T-1}\omega_t\sumov \DvtoD \EX \norm{\xv-\xbar}.}\nonumber
\end{align}
By replacing result of lemma \ref{lem3} we get
\begin{align}
\sum_{t=0}^{T-1}\frac{\omega_t}{\eta}\EX f(\xbar[{t+1}]) \leq \sum_{t=0}^{T-1}\frac{\omega_t}{\eta}\EX f(\xbar) + \frac{L\eta\rho\sigma^2}{2} \sum_{t=0}^{T-1}\omega_t - \frac{1}{4}\sum_{t=0}^{T-1}\omega_t \EX\norm{\nabla f(\xbar)}+ \\
&& \mathllap{L^2\big(\frac{1}{8L^2}\sum_{t=0}^{T-1} \omega_t  \EX\norm{\nabla f(\xbar)}+90\eta^2 (H+E) \big( \zeta^2 (H+E) + \sigma^2\big)  \sum_{t=0}^{T-1}  \omega_t\big).}\nonumber
\end{align}
By dividing both sides by $W_T = \sum_{t=0}^{T-1} \omega_t$, and rearranging, we have
\begin{align}
\frac{1}{8W_T}\sum_{t=0}^{T-1}\omega_t \EX\norm{\nabla f(\xbar)} \leq \frac{1}{W_T\eta}\sum_{t=0}^{T-1}\big(\omega_t\EX f(\xbar)-\omega_t\EX f(\xbar[{t+1}]) \big)+ \frac{L\eta\rho\sigma^2}{2} + \\
&& \mathllap{90L^2\eta^2 (H+E) \big( \zeta^2 (H+E) + \sigma^2\big).}\nonumber
\end{align}
\textbf{(non convex case)}

Combining this inequality with Lemma \ref{lem5} and considering that $\eta_t=\eta \leq \frac{1}{30L(H+E)}$, provides
\begin{align}
    \frac{1}{T}\sum_{t=0}^{T-1}\EX\norm{\nabla f(\xbar)} &\leq O \bigg(\frac{(f(\xz)-f^*)L(H+E)}{T} + \frac{\sigma\sqrt{\rho L(f(\xz)-f^*)}}{\sqrt{T}} + (\frac{L(f(\xz)-f^*)\sqrt{(H+E)(\sigma^2 + \zeta^2(H+E))}}{T})^{\frac{2}{3}} \bigg).
\end{align}

This completes the proof of Theorem \ref{T1}.

\end{document}